\documentclass[9pt,a4j]{article}
\usepackage{setspace}
\usepackage{amsmath,amsthm,bm,amssymb,extarrows,ascmac}
\usepackage{graphicx}
\usepackage{tikz}
\usetikzlibrary{intersections,shapes.arrows,calc}
\usepackage{booktabs}
\usepackage{multirow}
\usepackage{wrapfig}

\newtheorem{theo}{Theorem}[section]

\newtheorem{lemm}[theo]{Lemma}%
{\tiny }%

\newcommand{\argmin}{\mathop{\rm argmin}\limits}

\newcommand{\kf}[1]{\textcolor{black}{#1}}
\newcommand{\st}[1]{\textcolor{black}{#1}}
\newcommand{\kff}[1]{\textcolor{black}{#1}}

\usepackage{algorithm}
\usepackage{algorithmic}

\usepackage{amsmath,amsthm,bm,amssymb,extarrows,ascmac}
\usepackage{tikz}
\usetikzlibrary{intersections,shapes.arrows,calc}
\usepackage{cancel}

\def\bZ{\mathbb Z}

\def\bN{{\mathbb N}}
\def\bZ{{\mathbb Z}}
\def\bR{{\mathbb R}}

\def\X{{\mathcal{X}}}
\def\Y{{\mathcal{Y}}}
\def\Z{{\mathcal{Z}}}

\def\H{{\mathcal{H}}}

\def\E{{\mathcal{E}}}

\def\P{{\mathcal{P}}}
\def\D{{\mathcal{D}}}
\def\R{{\mathcal{R}}}
\def\N{{\mathcal{N}}}
\def\I{{\mathcal{I}}}
\def\E{{\mathcal{E}}}
\def\a{{\alpha}}

\def\e{{\varepsilon}}

\title{\ Invariance Learning based on Label Hierarchy}%
\author{Shoji Toyota\thanks{The Graduate University for Advanced Studies (SOKENDAI), Japan.}, Kenji Fukumizu \thanks{The Institute of Statistical Mathematics, Japan.}}%
\date{}%

\begin{document}
    \maketitle
    
\begin{abstract}

Deep Neural Networks inherit spurious correlations embedded in training data and hence may fail to predict desired labels on unseen domains (or environments), which have different distributions from the domain used in training.  Invariance Learning (IL) has been developed recently to overcome this shortcoming; using training data in many domains, IL estimates such a predictor that is invariant to a change of domain.  However, the requirement of training data in multiple domains is a strong restriction of IL, \kf{since it often needs high annotation cost}.  We propose a novel IL framework to overcome this problem.  Assuming the availability of data from multiple domains for a {\em higher} level of classification task, \kf{for which the labeling cost is low,} we estimate an invariant predictor for the target classification task with training data in a {\em single} domain.  Additionally, we propose two cross-validation methods for selecting hyperparameters of invariance regularization to solve the issue of hyperparameter selection, \kf{which has not been handled properly in existing IL methods}.  The effectiveness of the proposed framework, including the cross-validation, is demonstrated empirically, and the correctness of the hyperparameter selection is proved under some conditions. \end{abstract} 


\section{Introduction}\label{sec.Intro}

Training data used in machine learning unintentionally contain unrelated factors to \kff{the objective of the task,}
which are called {\em spurious correlations}.
Deep Neural Networks (DNNs) often inherit the spurious correlations embedded in the data in training domains and hence may fail to predict desired labels in  domains which have different distributions from the training domains, namely {\em unseen} domains. For example, in a problem of classifying images of cows, DNNs
\kf{tend to misclassify cows in sandy beaches}; most training pictures are taken in green pastures and DNNs inherit context information in training \cite{S. Beery,Shane}.  Systems trained with medical data collected in one hospital do not generalize well to other health centers; DNNs unintentionally 
\kf{extract} environmental factors specific to a particular hospital in training \cite{Ehab,Christian,Will}.


Invariance Learning (IL) is an approach 
developed \kff{rapidly} to overcome 
the issue of spurious correlation or short-cut learning
\cite{Martin,Kartik,Rothenhausler,Christina,Jonas,koyama,David, Liu, Liu_2,Creager,Parascandolo}. Let $e \in \E$ be a domain (or environment) index, and $X^e \in \X$ and $Y^e \in \Y$ be an input object and its label in domain $e \in \E$, \kff{respectively}. 
Using training data in \kf{{\em multiple}} training domains $\E_{tr} \subset \E$, IL estimates 
a predictor $f:\X \rightarrow \Y$
 that performs as well in unseen domains $e \in \E - \E_{tr}$ as in the training domain, that is, a predictor {\em invariant} to  change of domains. In IL, the training data must be annotated with labels of $\Y$ completely.


While the IL approach has attracted much attention as a solution to the \kf{spurious correlation}, 
requiring training data with exact labeling from multiple domains may hinder wide applications in practice; preparing training data in many domains are often expensive 
\kff{with} data annotation. 
Even when we can draw \st{data from multiple domains, they are often only available in an incompletely labeled form, which  are called }
{\em pseudo label data } \cite{Pham,Lee,Zheng, Gu}, {\em partial label data } \cite{Cour,Yan, Xu} and {\em complementary label data} \cite{Katsura, Feng, Ishida}.

We propose a novel IL framework to reduce the \kf{need} of training data with exact labels from multiple domains. We estimate the invariance by the 
\kff{coarser} 
labeled data, which need lower annotation cost.
\kff{Specifically}, in addition to the target task of classification, we consider \kff{another classification task of a higher level} in the label hierarchy, that is, we suppose that the additional task has a coarser label set than the target task.  We consider the situation where the training data of the target task is given in only {\em one} domain, while the task of higher label hierarchy has data 
from multiple domains.  
This significantly reduces the \kff{annotation cost} 
since the data with coarser labels are much easier to obtain.  
More formally, assuming the availability of additional data $(X^e,Z^e)$ from multiple domains $e\in\E_{ad} \subset \E$ for the higher level of classification task, we estimate an invariant predictor for the target task with training data of a {\em single} domain $\{ e^* \}$. Here, the higher level of classification task means that its label is given by $Z^e = g(Y^e)$ with some surjective function $g: \Y \rightarrow \Z$ \kf{to define the label hierarchy}. For example, consider the case where $\Y := \{bird_1,.....,bird_{10000},no~birds \}$, labels with 10,000 kinds of birds and no birds, and $\Z := \{there ~are~birds, no~birds\}$, $g(y)= no~birds$ if $y = no~birds$ and $g(y)= there~are~birds$ \kff{otherwise}. Then, the binary labels $\Z$ are much more easily annotated than the original $10001$ class labels $\Y$.  \kf{This may be done by humans with crowd-sourcing or a binary classifier of high accuracy.}


Another important issue in IL is hyperparameter selection.  Most IL methods involve some hyperparameters to balance the classification accuracy and the degree of invariance.  As \cite{David} and \cite{Ishaan} point out, in the literature of IL, the best performances of invariance have \kff{often} been achieved by selecting the hyperparameters using test data from unseen domains. Moreover, \cite{Ishaan} numerically demonstrated that, without using test data, simple methods of hyperparameter selection fail to find a preferable hyperparameter. This illustrates a strong need for establishing an appropriate 
method \kff{of hyperparameter selection} for IL. 

\kf{
We propose two methods of cross-validation (CV) for hyperparameter selection in our new IL framework.  Since we assume training data of a single domain for the target task, it is impossible to estimate the deviation of the risks over the domains.  We make use of the additional data from multiple domains in the higher level and provide methods of CV, which is applicable in the current setting. 
}
Theoretical analysis reveals that our methods select hyperparameter correctly with some condition. 
 
The main contributions of this paper are as follows:
\vspace{-1mm}
\begin{itemize}
\item  We establish a \kf{novel} framework of invariant learning, which estimates an invariant predictor from a single domain data, assuming additional data from multiple domains for a higher level of classification task.
\item We propose two methods of cross-validation under the framework for selecting hyperparameters without accessing any samples from unseen target domains.
\item Experimental studies verify that the proposed framework extracts an invariant predictor more effectively than other existing methods. 
\item We mathematically prove that the proposed CV methods select the correct hyperparameter under some \kff{settings}.
\end{itemize}

\section{Invariance Learning based on Label Hierarchy}\label{sec.HIT-Learning}
\paragraph{Notations} Throughout this paper, the space of input features and finite class labels are denoted by $\X$ and $\Y$, respectively.  For given predictor $f: \X \rightarrow \Y$ and random variable $(X,Y)$ on $\X \times \Y $ with its probability $P_{X,Y}$, $\R^{(X,Y)} (f)$ denotes the risk of $f$ on $(X, Y)$; $i.e.$, $\R^{(X,Y)}(f) := \int l(f(x),y) dP_{X,Y}$, where $l:\Y \times \Y \rightarrow \bR$ is a loss function.  For $m \in \bN_{>0}$, $[m]$ denotes the set $\{1,...,m \}$. For a finite set $A$, $|A| \in \bN$ denotes the number of elements in $A$.

\subsection{Review of Invariance Learning}\label{sec.setting}


Following \cite{Martin}, to formulate the out-of-distribution (o.o.d.) generalization, we assume that the joint distribution of data $(X^e,Y^e)$ depends on the domain (or environment) $e\in \E$, and consider the dependence of a predictor $f$ on the domain variable $e$.
Suppose we are given training data sets $\D^e:= \{(x^e_i, y^e_i) \}_{i=1}^{n^e} \sim P_{X^e, Y^e}$ \kf{i.i.d.}~from multiple domains $e \in \E_{tr}$. The final goal of the o.o.d.~problem is to predict a desired label $Y^e \in \Y$ from $X^e \in \X$ for larger target domains $\E \supset \E_{tr}$. 
To discuss the o.o.d.~performance, \cite{Martin} introduced the o.o.d.~risk
 \begin{equation}\label{eq:OOD-Bayes rule}
  \R^{o.o.d.}(f) := \max_{e \in \E} \R^e(f), 
 \end{equation}
where $\R^e (f):=  \R^{(X^e, Y^e)} (f)$.  This is the worst case risk over the domains $\E$, including unseen domains $\E- \E_{tr}$. 


To solve (\ref{eq:OOD-Bayes rule}), \cite{Martin} estimates such a predictor $w \circ \Phi$ that is invariant to change of domains, where the invariance  $\Phi: \X \rightarrow \H$ and the predictor $w :\H \rightarrow \Y$ function as eliciting the invariant representation from $X^e \in \X$, and  estimating a label of the invariant representation $\Phi(X^e)$ respectively. The estimation are implemented by solving the following optimization problem:
\begin{equation}\label{eq:IRM}
 \min_{\Phi \in \I_{tr}, \\
w: \H \rightarrow \Y}  \sum_{e \in \E_{tr}} \R^{e}(w \circ \Phi ),
\end{equation}
where $\I_{tr}$ is the set of invariances captured by $\bigcup_{e \in \E_{tr}} \D^e$.
All of IL, including \cite{Martin},  estimate the invariance \kf{using} the difference among $\E_{tr}$, assuming the availability of multiple training domains in common.

While how to capture the invariance varies slightly among IL, we adopt the method based on  conditional independence as done by \cite{Rojas}, \cite{Jonas} and \cite{koyama}:
$$\I_{tr} :=\bigl\{ \Phi : \X \rightarrow \H \left| \right. \bigl.P(Y^e | \Phi(X^e) ) \text{ does not depend on  }e \in \E_{tr}  \bigr\}.$$
\subsection{Invariance estimation by higher label data}\label{sec.HIT-Learning}

Our goal is to make an invariant predictor from a single training domain $\E_{tr} = \{ e^* \}$. 
In this case, (\ref{eq:IRM}) is reduced to 
the empirical risk minimization $\min_{f} \R^{e^*}(f )$ on $e^*$, and therefore the standard IL is not able to extract invariance.  

In this paper, we introduce an assumption that additional data $\D^e_{ad}$ from $(X^e, Z^e)$ with coarser label $Z^e\in \Z$ is available for multiple domains $\E_{ad} \subset \E$.  This means that we have data for another task in a higher label hierarchy than the target task.  More formally, the label $Z^e$ 
\kff{is assumed to follow} $Z^e = g(Y^e)$ with surjective label mapping $g: \Y \rightarrow \Z$ from the lower to the higher level in the hierarchy.  \kf{For example, in the problem of animal recognition from images, $\Y$ is $\{ \text{cow}, \text{horse}, \text{dog}, \ldots, \text{no animal}\}$ and $\Z$ may be $\{\text{animal}, \text{no animal}\}$.  The domain $e$ specifies the background of the image: pasture, sand, room, and so on.}  
 

By making use of the data $\bigcup_{e \in \E_{ad}} \D^e_{ad}$ in the higher level, our objective 
for the invariant prediction of $Y^e$ is given by 
\begin{equation}\label{eq:ILLH}
 \min_{\Phi \in \I_{ad}, \\
w: \H \rightarrow \Y}   \R^{e^*}(w \circ \Phi ),
\end{equation}
where $\I_{ad}$ is the set of invariances: 
$$\I_{ad} :=\bigl\{ \Phi : \X \rightarrow \H \left| \right. \bigl.
P(g(Y^e) | \Phi(X^e) ) \text{ does not depend on }e \in \E_{ad}  \bigr\}.$$
The condition of $\I_{ad}$ is necessary to the invariance condition of the target task $(X^e,Y^e)$,  that is,\\
\centerline{
``$P(Y^e | \Phi(X^e) )$ does not depend on $e \in \E$" $\Rightarrow$ }\\
\centerline{
``$P(g(Y^e) | \Phi(X^e))$ does not depend on  $e \in \E_{ad}$"}
\\
holds. 

\subsection{Construction of objective function}\label{sec.Objective}

While there are many possible variations in the design of losses and models, we \kff{focus} 
a probabilistic output case and evaluate its error by the cross entropy loss; that is, we model $w$ by maps to 
$p_{\theta}: \H \rightarrow \P_{\Y}$, where $\P_{\Y}$ denotes \kf{the set}
of probabilities on $\Y$. \kff{The risk is then written by}
$$
\R^{e}(p_{\theta} \circ \Phi) = \int   - \log  p_{\theta} ( Y^e | \Phi(X^e)) dP_{X^e, Y^e}.
$$

We aim to solve (\ref{eq:ILLH}) by minimizing the following objective function:
\begin{align*}
Objective(\theta, \Phi): =  \hat{\R} ^{e^*} (p_{\theta} \circ \Phi)  +\lambda \cdot \kff{(\text{Dependence measure of $P(Z^e | \Phi(X^{e}))$ on } e \in \E_{ad})}. 
\end{align*} 
Here, $\hat{\R}^{e^*} (p_{\theta} \circ \Phi)$ denotes the  empirical risk of $p_{\theta} \circ \Phi$ evaluated by $\D^{e^*}$: 
$$\hat{\R}^{e^*}(p_{\hat{\theta}} \circ \Phi) := - \frac{1}{| \D^{e^*} |} \sum_{(x^{e^*}, y^{e^*}) \in \D^{e^*}} \log p_{\hat{\theta}}(y^{e^*} | \Phi(x^{e^*}))$$.


\begin{algorithm*}[t]
\caption{Two Cross Validation Methods. If CORRECTION = True, $\lambda$ is selected by method II and if False, Method I.}
\label{alg:CV1}
\begin{algorithmic}[1]
\begin{spacing}{0.8}
\REQUIRE: Split $\D^{e^*},\D^{e_1}_{ad},...,\D^{e_n}_{ad}$ into $K$ parts. Set the hyperparameter candidates $\Lambda$.
\REQUIRE:$\hat{P}(Z^e =z^{\scalebox{0.2}{$\cancel{\hookrightarrow}$}}) \leftarrow \frac{{| \D_{ad,z^{\scalebox{0.2}{$\cancel{\hookrightarrow}$}} }^e |} }{| \D^e_{ad} |}$, where $D^e_{ad,z^{\scalebox{0.5}{$\cancel{\hookrightarrow}$}} } := \left\{ (x,z) \in \D^e_{ad} \left| z = z^{\scalebox{0.2}{$\cancel{\hookrightarrow}$}}\right. \right\}$ for all $e \in \kff{\E_{ad}}$ and $z^{\scalebox{0.2}{$\cancel{\hookrightarrow}$}}  \in \Z^{\scalebox{0.2}{$\cancel{\hookrightarrow}$}}$.
 \FOR {$\lambda \in \Lambda$}
 \FOR {$k= 1$ to $K$}
\STATE Learn $\theta_{[-k]}^{\lambda}, \Phi_{[-k]}^{\lambda}$ by using $\D^{e^*}_{[-k]}, \D^{e_1}_{ad,[-k]},...,\D^{e_n}_{ad,[-k]}$.
 \STATE  $\hat{\R}_{[k]}^{e^*}( p_{\theta_{[-k]}^{\lambda} }\circ\Phi_{[-k]}^{\lambda} )\leftarrow \frac{1}{|\D^{e^*}_{[k]} |} \sum_{(x^{e^*}, y^{e^*}) \in \D_{[k]}^{e^*}}- \log p_{\theta_{[-k]}^{\lambda}} (y^{e^*} | \Phi^{\lambda}_{[-k]}(x^{e^*}) )$ \\ \STATE $C (z^{\scalebox{0.2}{$\cancel{\hookrightarrow}$}}) \leftarrow \frac{1}{|\D^{e^*}_{[k], z^{\scalebox{0.2}{$\cancel{\hookrightarrow}$}} } |}\sum_{(x,y) \in \D^{e^*}_{[k], z^{\scalebox{0.3}{$\cancel{\hookrightarrow}$}} }} - \log p_{\theta_{[-k]}^{\lambda}}(y |  \Phi_{[-k]}^{\lambda}(x), g^{-1} (z^{\scalebox{0.2}{$\cancel{\hookrightarrow}$}} ) )$ for $z^{\scalebox{0.2}{$\cancel{\hookrightarrow}$}}$ in  $\Z^{\scalebox{0.5}{$\cancel{\hookrightarrow}$}}$.
  \FOR {$e$ $\in$  $\E_{ad}$}
 \STATE  $\hat{\R}_{[k]}^{(X^e, Z^e)}(p_{\theta_{[-k]}^{\lambda}} \circ \Phi_{[-k]}^{\lambda} )\leftarrow \frac{1}{|\D^{e}_{ad,[k]} |} \sum_{(x^{e}, z^{e}) \in \D_{ad,[k]}^{e}}- \log p_{\theta_{[-k]}^{\lambda}} (z^{e} | \Phi^{\lambda}_{[-k]}(x^{e}) )$.
  \STATE  $\hat{\R}_{[k]}^{e}(p_{\theta^{-k}_{\lambda}} \circ\Phi^{-k}_{\lambda} )\leftarrow \hat{\R}_{[k]}^{(X^e, Z^e)}(p_{\theta_{[-k]}^{\lambda}} \circ\Phi_{[-k]}^{\lambda}  ) +
  \sum_{z^{\scalebox{0.2}{$\cancel{\hookrightarrow}$}}  \in \Z^{\scalebox{0.2}{$\cancel{\hookrightarrow}$}}} \bigl\{ \hat{P}(Z^e = z^{\scalebox{0.2}{$\cancel{\hookrightarrow}$}} )  \cdot   C (z^{\scalebox{0.2}{$\cancel{\hookrightarrow}$}})  \bigr\}$.
 \ENDFOR
 \IF {CORRECTION}
 \STATE $\hat{\R}^{o.o.d.}_{[k]}( p_{\theta_{[-k]}^{\lambda}} \circ\Phi_{[-k]}^{\lambda} )\leftarrow \max \{ \max_{ e \in \E_{ad}}\hat{\R}_{[k]}^{(X^e, g(Y^e))}(p_{\theta_{[-k]}^{\lambda}} \circ\Phi_{[-k]}^{\lambda}  ),\hat{\R}_{[k]}^{e^*}( p_{\theta_{[-k]}^{\lambda}} \circ\Phi_{[-k]}^{\lambda})\}$
 \ELSE
 \STATE $\hat{\R}^{o.o.d.}_{[k]}( p_{\theta_{[-k]}^{\lambda}} \circ\Phi_{[-k]}^{\lambda} )\leftarrow \max_{e \in \E_{ad} \cup \{ e^* \} } \{\hat{\R}_{[k]}^{e}(p_{\theta_{[-k]}^{\lambda}} \circ\Phi_{[-k]}^{\lambda}  ) \}$
 \ENDIF
\ENDFOR
\STATE (I) $\hat{\R}^{o.o.d.} (\lambda) =\frac{1}{K}\sum_{k=1}^{K} \hat{\R}^{o.o.d.}_{[k]}( p_{\theta_{[-k]}^{\lambda}} \circ\Phi_{[-k]}^{\lambda}  )$
\ENDFOR
\STATE Select $\lambda^* := \argmin_{\lambda } \hat{\R}^{o.o.d.}(\lambda)$
\end{spacing}
\vspace{-0.1mm}
\end{algorithmic}
\end{algorithm*}

For the 
\kff{regularization} term, \kff{among others}, we adopt the one \kf{introduced in} \cite{Martin}, noting that their definition of invariance \kff{agrees} with ours in our setting, \kff{as we see below}.  \cite{Martin} call $\Phi$ invariant when $\argmin_{w} \R^e (w \circ  \Phi)$ is independent $e$, and to evaluate such an invariant feature, they propose the regularization term $\|  \nabla_{\hat{w} = w}\R^e(\hat{w}  \circ \Phi) \|^2$. The following \kf{lemma} ensures that our and their definition of the invariance \kff{coincide,}  
and therefore we may use $\sum_{e \in \E_{ad}} \|  \nabla_{\hat{w} = w}\R^e(\hat{w}  \circ \Phi) \|^2$ for our regularization term.
\begin{lemm}
When modeling $w$ by conditional probabilities, the following statements are equivalent:
\begin{itemize}
    \item[(i)] $P (Z^e | \Phi(X^e) )$ does not depend on $e \in \E$.
    \item[(ii)] $\argmin_{ p^{\Z | \H }_{\theta_{ad}}: \H \rightarrow \P_{\Z}} \R^{(X^e, Z^e)} ( p^{ \Z | \H  }_{\theta_{ad}}  \circ \Phi )$ does not depend on $e\in \E$,
\end{itemize}
\kff{where} $ p^{\Z | \H  }_{\theta_{ad}}$ in (ii) 
\kff{runs over} all probability densities.
\end{lemm}
\begin{proof}
Noting that $\argmin_{ \theta_{ad} } \R^{(X^e, Z^e)} (p_{\theta_{ad}} \circ \Phi)$ coincides with the probability density function of $P(Z^e | \Phi(X^e))$, the above equivalence follows immediately.
\end{proof}

In summary, we construct an objective function by 
\begin{align}\label{eq:completeobject}
&Objective(\theta,\theta_{ad} , \Phi): =   \hat{\R} ^{e^*} (p_{\theta} \circ \Phi)  +\lambda \cdot \sum_{e \in \E_{ad}} \|  \nabla_{\hat{\theta}_{ad} ={\theta}_{ad} }\hat{\R}^{(X^e, Z^e)}(p^{\Z | \H }_{\hat{\theta}_{ad}} \circ \Phi) \|^2,
\end{align}
where $p_{\theta}^{\Z | \H }: \H \rightarrow \P_{\Z}$ is  the logistic regression model, same as \cite{Martin}, and $\hat{\R}^{(X^e, Z^e)}(p^{\Z | \H }_{\hat{\theta}_{ad}} \circ \Phi) := - \frac{1}{| \D^e_{ad} |} \sum_{(x^e, z^e) \in \D^e_{ad}} \log p^{\Z | \H }_{\hat{\theta}_{ad}}(z^e | \Phi(x^e))$.

\section{Hyperparameter selection method}\label{sec.CV}

\subsection{Hyperparameter Selection in Invarance Learning}\label{subsec:diff}
Most IL methods involve hyperparameters to control the trade-off between the classification accuracy and invariance.  Thus, hyperparameter selection is essential in IL methods.  

The hyperparameter selection in IL has special difficulty.  
Since the o.o.d.~problem needs to predict 
$Y^e$ on unseen domains, 
the hyperparameter must be chosen 
without accessing any data in the unseen domains. 
\kf{It was reported \cite{Ishaan} that the success of IL methods depends strongly on the careful choice of hyperparameters; some of the results used data from unseen domains in the choice. }
\cite{Ishaan} reported also experimental results of various IL methods with two CV methods, training-domain validation (Tr-CV) and leave-one-domain-out validation (LOD-CV), and showed that the CV methods failed to select preferable hyperparamters.  In Colored MNIST experiment, for example, 
the accuracy of Invariant Risk Minimization \cite{Martin} is $52.0\%$ at best, which is about the random guess level. 


The failure of the CV methods are caused by the improper design of the objective function for CV; they do not simulate the o.o.d.~risk, which is the maximum risk over the domains.
Tr-CV splits data in each training domain into training and validation subsets, and takes the sum of the validated risks over the training domains. Obviously, this is not an estimate of the o.o.d.~risk. LOD-CV holds out one domain among the training domains in turn and validates models with the average of the validated risks over the held-out domains. Again, this average does not correspond to the o.o.d.~risk. 
In summary, the problem we need to solve is:
    How can we construct an evaluation function of  the o.o.d.~risk from validation data?
\kf{In the sequel, we will propose two methods of CV, which are summarized in Algorithm \ref{alg:CV1}.}

\subsection{Method I: \kff{using data of higher level task}}\label{subsec:CVmethod1}
We divide each of $\D^{e^*}, \D^{e_{1}}_{ad},...,\D^{e_{n}}_{ad}$ into $K$ parts, and  use the $k$-th sample $\{\D^{e^*}_{[k]}, \D^{e_1}_{ad,[k]},...,\D^{e_n}_{ad,[k]}\}$ and the rest $\{\D^{e^*}_{[-k]}, \D^{e_1}_{ad,[-k]},...,\D^{e_n}_{ad,[-k]}\}$  for validation and training, respectively.  To approximate the o.o.d.~risk of the predictor $p_{\theta_{[-k]}^{\lambda}} \circ \Phi^{\lambda}_{[-k]}$ obtained by the training set, we wish to 
estimate $R^{e}(p_{\theta_{[-k]}^{\lambda}} \circ \Phi^{\lambda}_{[-k]})$ for $e\in\E_{ad} \cup \{e^* \}$. 
For $e^*$, we use the standard empirical estimate $\hat{\R}^{e^*} (p_{\theta_{[-k]}^{\lambda}} \circ \Phi^{\lambda}_{[-k]}):= \frac{1}{| \D^{e^*}_{[k] } |}\sum_{(x^{e^*}, y^{e^*}) \in \D_{[k]}^{e^*}}- \log p_{\theta_{[-k]}^{\lambda}} (y^{e^*} | \Phi^{\lambda}_{[-k]}(x^{e^*}) )$.  For 
$e \in \E_{ad}$, we substitute unavailable $Y^e$ with $Z^e$ and use $\hat{\R}^{(X^e,Z^e)}(p_{\theta_{[-k]}^{\lambda}}  \circ \Phi_{[-k]}^{\lambda})  := \frac{1}{| \D^e_{ad,[k]} |} \sum_{(x^e,z^e) \in \D^e_{ad,[k]}}  - \log p_{\theta_{[-k]}^{\lambda}}(z^e | \Phi_{[-k]}^{\lambda}(x^e))$. 
\subsection{Method II: 
\kff{using correction term}
}\label{subsec:CVmethod2}

Method I can be improved by considering the difference $\R^{e} (p_{\theta_{[-k]}^{\lambda}} \circ \Phi^{\lambda}_{[-k]}) - \R^{(X^e, Z^e)} (p_{\theta_{[-k]}^{\lambda}} \circ \Phi^{\lambda}_{[-k]}) $, \kff{for which} 
we deduce the following theorem. The proof is given in Appendix \ref{append.proof_diff}.
\begin{theo}\label{thm:HIT-CV}
Let $\Z^{\scalebox{0.5}{$\cancel{\hookrightarrow}$}} := \left\{ z \in \Z \left| | g^{-1}(z) | > 1 \right. \right\}$.
For any map $\Phi: \X \rightarrow \H$, $p_{\theta} : \H \rightarrow \P_{\Y}$, and random variable $(X,Y)$ on $\X \times \Y$, the following equality holds:
\begin{align}
  &\R^{(X,Y)}( p_{\theta} \circ \Phi) - \R^{(X, g(Y))}( p_{\theta} \circ \Phi) =  \sum_{z^{\scalebox{0.3}{$\cancel{\hookrightarrow}$}}  \in \Z^{\scalebox{0.3}{$\cancel{\hookrightarrow}$}}} \Bigl\{ P(g(Y) = z^{\scalebox{0.3}{$\cancel{\hookrightarrow}$}} )  \times \int - \log p_{\theta}\bigl(Y | \Phi(X), Y \in g^{-1}(z^{\scalebox{0.3}{$\cancel{\hookrightarrow}$}} ) \bigr)dP_{(X,Y) | Y \in g^{-1}(z^{\scalebox{0.3}{$\cancel{\hookrightarrow}$}} )}\Bigl\}.
 \label{eq.decomp}
  \end{align}
  Here,  $P_{(X,Y)|  Y \in g^{-1}(z^{\scalebox{0.3}{$\cancel{\hookrightarrow}$}})}$ denotes the conditional distribution of $(X, Y)$ given the event $Y \in g^{-1}(z^{\scalebox{0.3}{$\cancel{\hookrightarrow}$}})$, and  $p_{\theta}(y | \Phi(x), Y \in g^{-1}(z^{\scalebox{0.3}{$\cancel{\hookrightarrow}$}} )):=\frac{p_{\theta}(y | \Phi(x) )}{\sum_{y \in g^{-1}(z^{\scalebox{0.3}{$\cancel{\hookrightarrow}$}}) } p_{\theta}(y | \Phi(x) )}$.
\end{theo}


The theorem shows that, to estimate 
$\R^{e} (p_{\theta_{[-k]}^{\lambda}} \circ \Phi^{\lambda}_{[-k]}) - \R^{(X^e, Z^e)} (p_{\theta_{[-k]}^{\lambda}} \circ \Phi^{\lambda}_{[-k]}) $, 
we need to estimate the following two values:
\begin{itemize}
\item[(i)]$P(g(Y^e) = z^{\scalebox{0.3}{$\cancel{\hookrightarrow}$}} ) =P(Z^e = z^{\scalebox{0.3}{$\cancel{\hookrightarrow}$}} )$ for every $z^{\scalebox{0.3}{$\cancel{\hookrightarrow}$}}\in \Z^{\scalebox{0.5}{$\cancel{\hookrightarrow}$}},$
\item[(ii)]$\int - \log p_{\theta_{[-k]}^{\lambda}}(Y^e | \Phi^{\lambda}_{[-k]}(X^e), g^{-1}(z^{\scalebox{0.3}{$\cancel{\hookrightarrow}$}} ) )dP_{X^e,Y^e | Y^e  \in  g^{-1}(z^{\scalebox{0.3}{$\cancel{\hookrightarrow}$}} )}$  for every $z^{\scalebox{0.3}{$\cancel{\hookrightarrow}$}} \in \Z^{\scalebox{0.5}{$\cancel{\hookrightarrow}$}}$. 
\end{itemize}


(i) is naturally estimated by 
$\D^e_{ad}$:$\hat{P}(Z^e = z) := \frac{{| \D_{ad,z^{\scalebox{0.3}{$\cancel{\hookrightarrow}$}} }^e |} }{| \D^e_{ad} |}$, where $D^e_{ad,z^{\scalebox{0.3}{$\cancel{\hookrightarrow}$}} } := \left\{ (x,z) \in \D^e_{ad} \left| z = z^{\scalebox{0.5}{$\cancel{\hookrightarrow}$}}\right. \right\}$. 
\kf{(ii) is not directly \kff{estimable}, as $Y^e$ is not observed.  For an approximation of (ii), }
we \kf{substitute $Y^e$ with $Y^{e^*}$ and use } 
 \begin{align*}
 \frac{1}{| \D^{e^*}_{[k], z^{\scalebox{0.3}{$\cancel{\hookrightarrow}$}} } |}\sum_{(x,y) \in \D^{e^*}_{[k], z^{\scalebox{0.3}{$\cancel{\hookrightarrow}$}} }} - \log &p_{\theta_{[-k]}^{\lambda}}(y |  \Phi_{[-k]}^{\lambda}(x), g^{-1} (z^{\scalebox{0.3}{$\cancel{\hookrightarrow}$}} ) ),
  \end{align*}
where $$  \D^{e^*}_{[k], z^{\scalebox{0.3}{$\cancel{\hookrightarrow}$}} } := \left\{ (x,y) \in \D^{e^*}_{[k]} \left| g(y) = z^{\scalebox{0.3}{$\cancel{\hookrightarrow}$}}  \right. \right\} \subset \D^{e^*}_{[k]}.$$ 

\subsection{Theoretical analysis of our cross validation methods} \label{subsec:theo}
In Sections \ref{subsec:CVmethod1} and \ref{subsec:CVmethod2}, we approximate $\R^{(X^e,Y^e)}$ using $Z^e$.  While the approximation is not exact, we will \kf{prove} that the proposed CV methods still \kf{select a correct hyperparameter} under some conditions.  
We will also elucidate the difference of the two CV methods theoretically. 
In this paper, \kf{to avoid discussing non-trivial effects of nonlinear $\Phi$,} we focus the case of variable selection, which appears in the invariance problem of causal inference \cite{Jonas,Christina} and regression \cite{Rojas}. 

Let $\X:= \X_1 \times\X_2$ where $\X_1 := \bR^{n_1}$ and $\X_2 := \bR^{n_2}$ with $n_1,n_2 \in \bN$, and assume that the projection $\Phi^{\X_1}$ from $\X$ onto $\X_1$ yields the invariance to $\{(X^e, Y^e) \}_{e \in \E}$, that is, $P(Y^e | \Phi^{\X_1} (X^e))$ are the same over $\forall (X^e, Y^e)\in \{(X^e, Y^e) \}_{e \in \E}$. \kf{Suppose that we have a model $\{ \Phi \}$ for the feature map, where $\Phi$ is a projection of $\X$ onto some subset of $n_1+n_2$ variables of $\X$.  Minimizing (\ref{eq:completeobject}) with its hyperparameter $\lambda$ over the model yields the feature map as a  projection $\Phi^{\lambda}: \X \rightarrow \bR^{n_{\lambda}}$ ($n_{\lambda} \leq n_1+ n_2$). }
Let $\Phi_2^{\lambda}$ denote the $\X_2$-component of $\Phi^{\lambda}$ and, if $\Phi^{\lambda}$ does not have any $\X_2$-component, we write $\mathrm{Im} \Phi^{\lambda}_2= \emptyset$.  
\kf{For simplicity of theoretical analysis, we assume that minimization with its hyperparameter $\lambda$ learns perfectly $p^{*, \lambda}(y | \Phi^{\lambda}(x))$, the conditional probability density function of $P_{Y^{e^*} | \Phi^{\lambda}(X^{e^*})}$.} 
Then, \kf{neglecting the estimations}, the approximated o.o.d.~risk of $p^{*, \lambda}  \circ \Phi^{\lambda}$
used by Methods I and II 
are represented by the following $\R^{I} (\lambda)$ and $\R^{II} (\lambda)$, respectively:


$$\R^{I}(\lambda) := \max \Biggl\{   \max_{e \in \E_{ad}}   \R^{(X^e,g(Y^e))}(p^{*, \lambda}  \circ \Phi^{\lambda}),R^{(X^*,Y^*)}(p^{*, \lambda}  \circ \Phi^{\lambda}) \Biggr\}, $$
\begin{align*}
\R^{II} (\lambda) :=  \max_{e \in \E_{ad} \cup \{ e^* \} } \Biggl\{   \R^{(X^e,g(Y^e))}(p^{*, \lambda} \circ \Phi^{\lambda})+ \hspace{-3mm}\sum_{z^{\scalebox{0.5}{$\cancel{\hookrightarrow}$}} \in \Z^{\cancel{\hookrightarrow}}}  \Bigl\{  P(g(Y^e)= z^{\scalebox{0.5}{$\cancel{\hookrightarrow}$}} )  
\times   \int-\log p^{*, \lambda} (Y^{e^*} | \Phi^\lambda(X^{e^*}),  Y^{e^*} \in  g^{-1}(z^{\scalebox{0.3}{$\cancel{\hookrightarrow}$}}))& \\
~~~~~~~~~~~~~~~~~~~~~~~~~~~~~~~~~~~~~~~~~~~~~~dP_{(X^{e^*},Y^{e^*}) | Y^{e^*} \in  g^{-1}(z^{\scalebox{0.3}{$\cancel{\hookrightarrow}$}})}  \Bigl\}  \Biggl\}. &&
\end{align*}\label{eq.HOOD}
%
We have the following theoretical justification of our CV methods: \kff{the chosen $\lambda$ gives a minimizer of the correct CV criterion.} For the proofs, see Appendices \ref{append:CV-theo} and \ref{append.proof_cv2}.

\begin{theo}[Effectiveness of Method I]\label{Thm:justfyHITCV}
Assume that the following four conditions  of $\E$, $\E_{ad}$ and $\Lambda$ hold:
\begin{itemize}
\item[(a)] $P_{\Phi^{\X_1}(X^e)}$ does not depend on $e\in\E$.
\item[(b)] For any random variable $(X, Y)$ with $P_{\Phi^{\X_1}(X),Y } =P_{\Phi^{\X_1}(X^{e^*}),Y^{e^*} } $, there exists $e \in \E$ such that $(X,Y )=(X^e,Y^e ) $.
\item[(c)]  $\exists \lambda^{I} \in \Lambda$ s.t. $\Phi^{\lambda^{I}}=  \Phi^{\X_1}$.
\item[(d)]$\forall \lambda$ with $\mathrm{Im} \Phi^{\lambda}_2 \neq \emptyset$,  $\exists e_{\lambda} \in \E_{ad}$ such that $P\Bigl( g(Y^{e^*})   | \Phi^{\lambda}(X^{e^*})  \Bigr) \leq e^{-\beta} - \e$, $P_{X^{e_\lambda}, Y^{e_\lambda}}$-a.e.  $i.e.$, $P_{X^{e_\lambda}, Y^{e_\lambda} } ( B ) = 1$, where
\end{itemize}
\begin{equation*}
B:=\left\{ (x, y) \in \X \times \Y \left|
\begin{array}{l}
P\Bigl( g(Y^{e^*})= g(y) |\Phi^{\lambda}(X^{e^*}) =\Phi^{\lambda}(x)  \Bigr) \\
~~~~~~~~~~~~~~~~~~~~~~~~~~~~~~~~~~~~~~\leq e^{-\beta} - \e 
\end{array}
\right.\right\}.
\end{equation*}
  
Here $\beta := H(Y^{e^*} | \Phi^{\X_1}(X^{e^*}))$
is the conditional entropy of $(\Phi^{\X_1}(X^{e^*}), Y^{e^*})$.

Then, \kf{we have}
$$\argmin_{\lambda \in \Lambda} \R^{I}(\lambda) \subset \argmin_{\lambda \in \Lambda} \R^{o.o.d.}(p^{*, \lambda}  \circ \Phi^{\lambda}).$$ 
\end{theo}
\begin{theo}[Effectiveness of Method II]\label{Thm:justfyHITCV2}
Assume that (a), (b), (c) and the following condition (d)' hold:
\begin{itemize}
\item[(d)'] $\forall \lambda$ with $\mathrm{Im} \Phi^{\lambda}_2 \neq \emptyset$,  $\exists e_{\lambda} \in \E_{ad}$ such that $P\Bigl( g(Y^{e^*})   | \Phi^{\lambda}(X^{e^*})  \Bigr) \leq e^{-\beta_{\lambda}} - \e$, $P_{X^{e_\lambda}, Y^{e_\lambda}}$-a.e.
\end{itemize}
\begin{align*}
~~\beta_{\lambda} := H(Y^{e^*} | \Phi^{\X_1}(X^{e^*})) -   & \sum_{z^{\scalebox{0.3}{$\cancel{\hookrightarrow}$}}  \in \Z^{\cancel{\hookrightarrow}}}\Bigl\{ P(g(Y^{e^*}) = z^{\scalebox{0.3}{$\cancel{\hookrightarrow}$}} ) 
\times \int - \log p^{*, \lambda}(Y^{e^*} | & \Phi^{\lambda}(X^{e^*}), g^{-1}(z^{\scalebox{0.3}{$\cancel{\hookrightarrow}$}} ) )
&dP_{(X^{e^*},Y^{e^*}) | Y^{e^*} \in g^{-1}(z^{\scalebox{0.3}{$\cancel{\hookrightarrow}$}} )} \Bigl\}.
\end{align*}
Then, \kf{we have}
$$\argmin_{\lambda \in \Lambda} \R^{II}(\lambda) \subset \argmin_{\lambda \in \Lambda} \R^{o.o.d.}(p^{*, \lambda}  \circ \Phi^{\lambda}).$$ 

\end{theo}

The assumptions (d) and (d)' mean that 
labeling rules on $\E_{ad}$ and $e^*$ are different due to a domain-specific factor ($i.e.$, an $\X_2$-component). (d) and (d)' mean that, if $\lambda$ fails to remove environment factors ($i.e.$, $\mathrm{Im} \Phi^{\lambda}_2 \neq \emptyset$), for some $e_{\lambda} \in \E_{ad}$, 
either of the following two inequalities $$P(g(Y^{e^*})   = g(y)| \Phi^{\lambda}(X^{e^*})  = \Phi^{\lambda}(x) ) \leq e^{-\beta}  - \e$$
$$P(g(Y^{e^*})   = g(y)| \Phi^{\lambda}(X^{e^*})  = \Phi^{\lambda}(x) ) \leq  e^{-\beta_{\lambda}} - \e$$ 
\kff{hold on a set of $(x,y)$ with probability $1$ w.r.t.~$P_{( \Phi^{\lambda} (X^{e_{\lambda}}),Y^{e_{\lambda}})}$}. Noting that $( \Phi^{\lambda}(X^{e^*}), Y^{e^*})$ attaches a label $z$ to $\Phi^{\lambda}(x)$ such that $P\Bigl( g(Y^{e^*})= z   | \Phi^{\lambda}(X^{e^*}) = \Phi^{\lambda}(x)  \Bigr)$  is large, the inequalities in (d) and (d)' mean  that $Y^{e_\lambda}$ and $Y^{e^*}$ attach different labels to the same $x \in \X$ with high probability due to its $\X_2$-component.



The theoretical analysis shows, while Method I is simpler than Method II, Method II is more applicable than Method I; that is because Method II eases the sufficient condition (d) for Method I to succeed. Noting that $\beta \geq \beta_{\lambda}$ and hence, $e^{-\beta} - \e \leq e^{-\beta_{\lambda}} - \e$, the condition (d)' is milder than (d). The fact implies that we can apply Method II even when
labeling rules on $\E_{ad}$ and $e^*$ due to domain-specific factors are too similar to apply Method I.

We also show one of the sufficient  conditions of $e^*$ for there to exist $(X^{e_{\lambda}}, Y^{e_{\lambda}})$ which satisfies the inequalities in (d) or (d)'  for $\forall \lambda$ with $\mathrm{Im} \Phi^{\lambda}_2 \neq \emptyset$  in Appendix \ref{append:condi_(c)(c)'}. 
\section{Related work}\label{sec.Related works}

\paragraph{Fine-tuning} \kf{The proposed framework uses additional data from multiple domains as well as the training data for the target task so that it may be relevent to} Transfer learning (TL) \cite{Pan,Yang,Yosinski} and meta-learning \cite{Fin,Andrychowicz1}, which realize fast and accurate learning for a new target task based on a model pre-trained with additional data sets or tasks. For example, after initial learning with a large data set, {\em fine tune} \cite{Pan,Yang} re-trains the model with the target task, while {\em frozen feature} \cite{Yosinski} fixes the pre-trained model and tunes a head network.  Although they show advantages in many learning problems, they may not work effectively in the current setting; in the fine-tuning with the target task $(X^{e^*},Y^{e^*} )$, the model tends to learn  spurious correlation in the data set and does not generalize to o.o.d.~domains.  Some fine-tuning methods will be compared with the proposed approach experimentally in Section  \ref{sec.Experiment}.

\vspace{-3mm}
\paragraph{Domain adaptation by deep feature learning } 
Domain adaptation strategies by deep feature learning \cite{Yaroslav, Shai, Gilles,Stojanov,Zhang} assume that we can access input data on a test domain in advance, and try to obtain data representation $\Phi(X^e)$ that follows the same distribution for the training and test domains. While the strategies lead to high predictive performance on a test domain similar to a training domain, such $\Phi$ does not function by discarding environmental factors from $X^e \in \X$ as noted in \cite{Martin}.  Experimental comparisons will be shown in Section \ref{sec.Experiment}.


\section{Experiments}\label{sec.Experiment} 

  \begin{table*}[t]
  \vspace{-3mm}
  \caption{Average Test Accuracies and SEs of Synthesized Data on $e = -e^*$ (5 runs): {\em Oracle} shows the results of the experiments with the first component. The best scores are \textbf{bolded}.}\label{Table:Syndata}
  \vspace{3mm}
\centering
\resizebox{\textwidth}{!}{
 \scalebox{0.6}{ \begin{tabular}{|c|c|c|c|c|c|c|c|c|c|c|c|} \hline
   &   $e^*=0$ &  $e^*=5$ &  $e^*=10$ & $e^*=15$ & $e^*=20$ &  $e^*=25$ &  $e^*=30$ & $e^*=35$  &$e^*=40$ &  $e^*=45$ & $e^*=50$\\ \hline 
   Oracle& \multicolumn{11}{|c|}{906	(.007) } \\ \hline  \hline 
    ERM &  .789	(.218) &  .791 (.174) &.637 (.188) &.329 (.201) & .324 (.328) &  .311 (.260) &  .159	(.193) &  .140 (.171)  &.132 (.161) &  .166	(.147) &  .051 (.101)\\ \hline 
    FT& \textbf{.899 (.000)} &  \textbf{.863 (.001)}	&.575 (.002) &.568 (.001) & .673 (.103) &  .583 (.088) &  .402	(.004) &  .350 (.001)  &.003 (.000) &  .000 (.000) &  .000 (.000) \\\hline 
     FE & \textbf{.899 (.000)} &  .861 (.002)  & .540 (.102) & .568 (.001)&  .673 (.102) &.628 (.001)  &  .401	(.001)  &.351 (.002) &  .066 (.132)	  &.000	(.000) &  .000 (.000)\\ \hline
    DSAN & .684	(.008) &  .367 (.016)	 &.195 (.015) &.112 (.008)& .045 (.008) &  .013 (.003)& .006	(.001)  &.001(.001)   &.000 (.000) &  000	(.000) &  .000	(.000)\\\hline \hline
     Ours + Our CV I& .799	(.232) &  .784 (.231)	  &\textbf{.884 (.021) } & \textbf{.875 (.044)} &\textbf{ .815 (.098)}&  $\textbf{.738 (.209)}$ & \textbf{ .865 (.047)} &  \textbf{.659	(.233)  }&\textbf{.666 (.285)} &  \textbf{.776 (.080)} & \textbf{ .699 (.255) } \\ \hline 
     Ours + Our CV II& .799	(.232) &  .783 (.231)	  &\textbf{.884 (.021)}  & \textbf{.875 (.044) }&\textbf{ .815 (.098) }&  $\textbf{.738 (.209)}$ & \textbf{ .865	(.047)} & \textbf{.659 (.233)  }&.563 (.291) &\textbf{  .776	(.080)  }&\textbf{ .699	(.255) }\\  \hline \hline
     Ours + Tr-CV& .790	(.230) &  .776 (.225)	  &.609 (.163)  & .491 (.095) & .366 (.147)&  .248 (.192) & .376	(.033) &  .215 (.168)  &.148 (.127) &  .189	(.108) &  .031 (.138)\\ \hline 
     Ours + LOD-CV& .662	(.180) &  .521 (.145)	  &.569 (.204)  & .538 (.168) & .450 (.158)&  .371 (.213) & .641	(.221) &  .571	(.221)  &.380 (.196) &  .423	(.218) &   .316	(.127)\\ \hline \hline
     Ours + TDV& .915	(.005) &  .905 (.006)	  &.896 (.002)  & .895 (.010) & .848 (.059)&  .849  (.069) & .887	(.030) &  .764	(.152)  &.796 (.174) &  .848	(.055) &  .775	(.179)\\ 
     \hline
     
  \end{tabular}}}\
 \vspace{-3mm}
\end{table*}

  \begin{table*}[t]
  \caption{Comparison of Two CV methods. Average Test Accuracies and SEs of the estimates (10runs). Best scores are \textbf{bolded}. }\label{Table:Syndata2}
  \vspace{3mm}
\centering
\resizebox{\textwidth}{!}{
 \scalebox{0.6}{ \begin{tabular}{|c|c|c|c|c|c|c|c|c|c|c|c|} \hline
   &   $e_{ad}=-9$ &  $e_{ad}=-8$ &  $e_{ad}=-7$ & $e_{ad}=-6$ & $e_{ad}=-5$ &  $e_{ad}=-4$ &  $e_{ad}=-3$ & $e_{ad}=-2$  &$e_{ad}=-1$ &  $e_{ad}=0$ & $e_{ad}=1$\\ \hline 
   TDV   &  .596	(.078) &  .621	(.046) &.630 (.041) &.595 (.061) & .590 (.087) &  .621 (.059) &  .564 (.071) &  .582 (.056)  &.535 (.093) &  .520	(.121) &  .575 (.107)\\ \hline  \hline
    CV I & \textbf{.529	(.128)} &  .555	(.111) &.562 (.086) &.566 (.109) & .375 (.145) &  .346 (.172) &  .372 (.176) &  .358 (.167)  &.300 (.146) &  .173 (.143) &  .218	(.087)\\ \hline 
     CV II & .527 (.152) & \textbf{ .573 (.089)	 } &\textbf{.565 (.085) } &\textbf{ .572 (.072)} &\textbf{ .522 (.110)}&  $\textbf{.523 (.102})$ & \textbf{ .482 (.113)}  & \textbf{ .506 (.153) }&\textbf{.430 (.146)}& \textbf{ .437 (.157)} & \textbf{ .502 (.149) }\\  \hline 
      \end{tabular}}}\
 \vspace{-1mm}
\end{table*}

We study the effectiveness of the proposed framework and CV methods through experiments, comparing them with several conventional methods: empirical risk minimization (ERM), fine-tuning methods, and deep domain adaptation strategies. As fine-turning methods, we employ two typical types of transfer learning: {\em{fine tune}} (FT) and {\em{frozen feature}} (FF) \cite{Pan,Yang,Yosinski}.  As a deep domain adaptation  technique, we adopt the 
state-of-the-art method $DSAN$ \cite{Stojanov}.  We also compare our two CV methods (CVI and CVII) with conventional CV methods: training-domain validation (Tr-CV) and leave-one-domain-out cross-validation (LOD-CV) \cite{Ishaan}.  \kf{We have two hyperparameters to be selected by CV.}  In the training with \eqref{eq:completeobject}, we set $\lambda
=1$ when the training epoch is less than a certain threshold $t$, and $\lambda=\lambda_{after}$ if the epoch is larger than $t$. From a set of candidates, each of the CV methods selects a pair $(t,\lambda_{after})$. To know the best possible performance among the candidates, we also apply the  test-domain validation (TDV) \cite{Ishaan}, which  selects the hyperparameters with the unseen test domain. Note that TDV is not applicable in practical situations.  The details on the experiments can be found in Appendix \ref{app:add_exp_detal}. 

 \vspace{-2mm}
\paragraph{Synthesized Data 1} We compared the proposed method with the other approaches using synthesized data with $\X = \bR^2$, $\Y = [3]$ and  $\Z := [2]$.  We used distributions $N_0:=  \mathcal{N} (0, 10^2) \times \mathcal{N} (e, 10^2)$, $N_1:= \mathcal{N} (30, 10^2) \times \mathcal{N} (-4e, 10^2)$ and $N_2:=\mathcal{N}(-30, 10^2) \times \mathcal{N} (-e, 10^2)$, where $\N (a, b)$ denotes a normal distribution with its (mean, variance) $=(a, b)$.  Given $x\sim N_i$, the task is to predict $N_i$ among $i=0,1,2$.  The aim of IL is to ignore the second component of $x$, as it works as an environmental factor.   Given $e^* \in \bN_{\geq 0}$ ranging from $0$ to $50$, each experiment draws $\D^{e^*} \sim P_{X^{e^*}, Y^{e^*} }$ with its sample size $n^{e^*}=2000$, and then predicts $Y^{-e^*}$ from $X^{-e^*}$. Setting $g$ by $g(0)=0$ and $g(1) = g(2)=1$, we draw $\D^e_{ad} \sim P_{X^e, Z^e }$  from $\E_{ad}=\{-100, -50,0,50,100 \}$  with its sample size $n^e=1000$ ($\forall e \in \E_{ad}$). We model $\Phi$ by a $3$-layer neural net. Setting the maximum epoch $500$, we select $(t, \lambda_{after})$ from $3 \times 5$ candidates with $t \in \{0,100,200\}
$ and $\lambda_{after} \in \{10^0,10^1,...,10^4\}$ by each of the CV methods.  Table \ref{Table:Syndata} shows the test accuracy of the estimates for $e = -e^*$ over 2000 random samples $(x, y) \sim P_{X^{-e^*}, Y^{-e^*}}$. When $e^*=0$ and $5$, 
The environmental bias of  training ($e^*$) are similar to the one of test $(-e^*)$, and hence, the fine-tuning methods yield high performances, which may use spurious correlation. As $e^*$ increases, the difference  between the training ($e^*$) and test $(-e^*)$ distributions becomes larger, and the previous methods fail to achieve high accuracy.  The proposed methods (Ours) keep higher performance than the others even for large $e^*$.  Among the CV methods, our two methods (CVI, CVII) significantly outperform the others for larger $e^*$.  For this data set, CVI and CVII show almost the same performance.  
 \vspace{-3mm}
 
\paragraph{Synthesized Data 2}  To highlight the difference of the proposed CVI and CVII, we compare them regarding the difference between domains $\E_{ad}$ and $e^*$. 
We used synthesized data with $\X = \bR^2$, $\Y = [10]$ and $\Z := [2]$,  preparing ten distributions $\{N_i \}_{i=1}^{10}$ on $\bR^2$, which include an environmental bias in the second component depending $e \in \bZ$ (see Appendix \ref{subsc:exp.rep.of.Syn} for explicit representations of $\{N_i \}_{i=1}^{10}$). The task is to predict $N_i$ ($i=1,\ldots,10$) for $x \sim N_i$.  Setting  $e^*:=20$ with $n^{e^*}=60000$, the test task is to predict the label for domain $e=-20$.  Regarding the task with label of higher level, we use $g(y)=0$ if $y$ is odd and $g(y)=0$ if $y$ is even.   We draw $\D^e_{ad} \sim P_{ X^e, Z^e }$  ($n^e=20000$) from $\E_{ad}=\{e_{ad}, 40 \}$, where $e_{ad}$ ranges from $-9$ to $1$.  As $e_{ad}$ increases, the domains $\E_{ad}$ approach to $e^*$, and thus 
the labeling rule on $\E_{ad}$ becomes similar to one of $e^*$.
The model $\Phi$ is a $3$-layer neural net.  We set the maximum epoch $500$, and select the hyperparameter $(t, \lambda_{after})$ from $4$ candidates with $t \in \{0\}$ and $\lambda_{after} \in \{0,0.001,80,100\}$  by each CV method. Table \ref{Table:Syndata2}
shows the test accuracy of the estimates for $e = -e^*$ over 2000 random samples $(x, y) \sim  P_{X^{-e^*}, Y^{-e^*}}$. The results show that, while CVI fails to select optimal hyperparameters as $e$ increases, CVII keeps higher performance, which accords with the theoretical implication in Section \ref{subsec:theo}.
 
\vspace{-3mm}

 \begin{table}[t]
\vspace{-8mm}
\vskip 0.2in
\begin{center}
 \caption{Average Test Accuracies and SEs of  Hierarchical Colored MNIST (5runs). TDV selects $\lambda$ which yields the highest performance on $e=0.9$. Best scores are \textbf{bolded}.}
 \label{Fig:acc Colored MNIST_tbl}
 \vspace{3mm}
 \renewcommand{\arraystretch}{1.1}
 \begin{tabular*}{8.3cm}{@{\extracolsep{\fill}}|c|c|c|} \hline
     & \small  $e=0.1~~~~~~$ &   \small  $e=0.9~~~~~~$  \\ \hline \hline
      \small   Best possible&  \multicolumn{2}{c|}{ \small .800 } \\ \hline \hline
   
  \small   Oracle (grayscale)&  \multicolumn{2}{c|}{ \small .780(.002) } \\ \hline
     \small ERM&\small .796 (.000) &\small .177 (.006) \\
    \small FT&\small$\textbf{ .800 (.001)}~~$ &\small .201 (.004) \\
   \small FE&\small .796 (.001) &\small .200 (.007) \\
    \small DSAN&\small .789 (.004) &\small .091 (.005) \\\hline\hline
    \small Ours +Our CV I&\small .773 (.003) &\small .644 (.011)\\ 
    \small Ours +Our CV II&\small .745 (.008) &\small $\textbf{ .707 (.012) }~~$ \\  \hline \hline
    \small Ours +Tr-CV&\small .794 (.004)  &\small .541 (.007) \\
    \small Ours +LOD CV&\small .338 (.048) &\small .334 (.029)\\ \hline \hline
    \small Ours +TDV&\small .738 (.018) &\small .732 (.008)\\\hline
  \end{tabular*}
  \vspace{2mm}
 \end{center}
\vskip -0.2in
\end{table}

\begin{table}[t]
\vspace{-6mm}
\vskip 0.2in
\begin{center}
 \caption{Means and SEs of $\{$({\text Accuracy of TDV on $e=0.9$}) -({\text  Accuracy of Each CV on $e=0.9$}) $\}$ (5runs). Lowest errors are \textbf{bolded}. }
 \label{Fig:acc Colored MNIST_CVdiff_tbl}
 \vspace{3mm}
 \renewcommand{\arraystretch}{1.}
 \begin{tabular*}{7.6cm}{@{\extracolsep{\fill}}|c|c|c|c|} \hline
      \footnotesize  CVI &   \footnotesize  CVII & \footnotesize Tr-CV &   \footnotesize LOD-CV  \\ \hline \hline
     \footnotesize  .088 (.004)  &\footnotesize $\textbf{.025 (.006)}$ & \footnotesize  .191 (.019)  &\footnotesize  .398 (.025)  \\ \hline
  \end{tabular*}
  \vspace{2mm}
 \end{center}
\vskip -0.2in
\end{table}

\paragraph{Hierarchical Colored MNIST} We apply our framework to $Hierarchical$ $Colored$ $MNIST$, which is an extended version of Colored MNIST \cite{Martin} with $\Y = [3]$ and  $\Z := [2]$.  We aim to predict $Y^e$ from digit image data $X^e$, which is in the three categories $0-2$ ($y =0$),  $3$ or $4$ ($y =1$) and $5-9$ ($y =2$).  
The label is changed \kf{randomly to one of the rest} with a probability of $20\%$, which is denoted by  $\hat{y}$. 
The environment $e$ controls the color of the digit; for $\hat{y}$= $0,1$, the digit is colored in red with probability $e$ and for $\hat{y}$= $2$ colored in green with probability $1-e$.  In the experiment, $\D^{e^*} \sim P_{X^{0.1}, Y^{0.1} }$ is drawn with sample size $n^{e^*} =5000$, and $Y^e$ is predicted based on $X^e$ for $e=0.1$ and $0.9$. Regarding $Z^e$, we consider the task where we predict $z=0$ for $X^e$ in $0-2$ and $z=1$ for $3-9$ (that is, $g(0)=0$ and $g(1)=g(2)=1$). We obtain the final label $\hat{z}$ by flipping $z$ with $20 \%$.  As the environment factor, we color the digit red for $\hat{z}$= $0$ with probability $e$ and green for $\hat{z}=1$ with probability $1-e$.  We set $\E_{ad} = \{ 0.1,0.3,0.5,0.7,0.9 \}$ with   $n^e=5000$ for $\forall e \in \E_{ad}$. We model $\Phi$ by a $3$-layer neural net. With the maximum epoch $500$, we select $(t, \lambda_{after})$ from $3 \times 10$ candidates with $t \in \{0,100,200\}, \lambda_{after} \in \{ 10^0, 10^1,...,10^9\}$ by each CV method. Table \ref{Fig:acc Colored MNIST_tbl} shows test accuracies for 2000 random samples in the environment $e=0.1$ and $e=0.9$.  The results, together with Appendix \ref{append:additionalMNIST}, demonstrate that the proposed methods significantly outperform the others for $e=0.9$.  Among the two proposed methods, CV II yields the higher test accuracy. Table \ref{Fig:acc Colored MNIST_CVdiff_tbl} shows the difference between accuracies by TDV and each CV for the same data set with $e=0.9$.  The results, together with Appendix \ref{append:additionalMNIST}, verify that CVII selects preferable hyperparameters with smaller errors.

\begin{figure*}
\vspace{-8mm}
\vskip 0.2in
\begin{center}
\caption{Visualization of Bird recognition problem}\label{fig:birds-vis}
\vspace{3mm}
\includegraphics[width=.6\columnwidth]{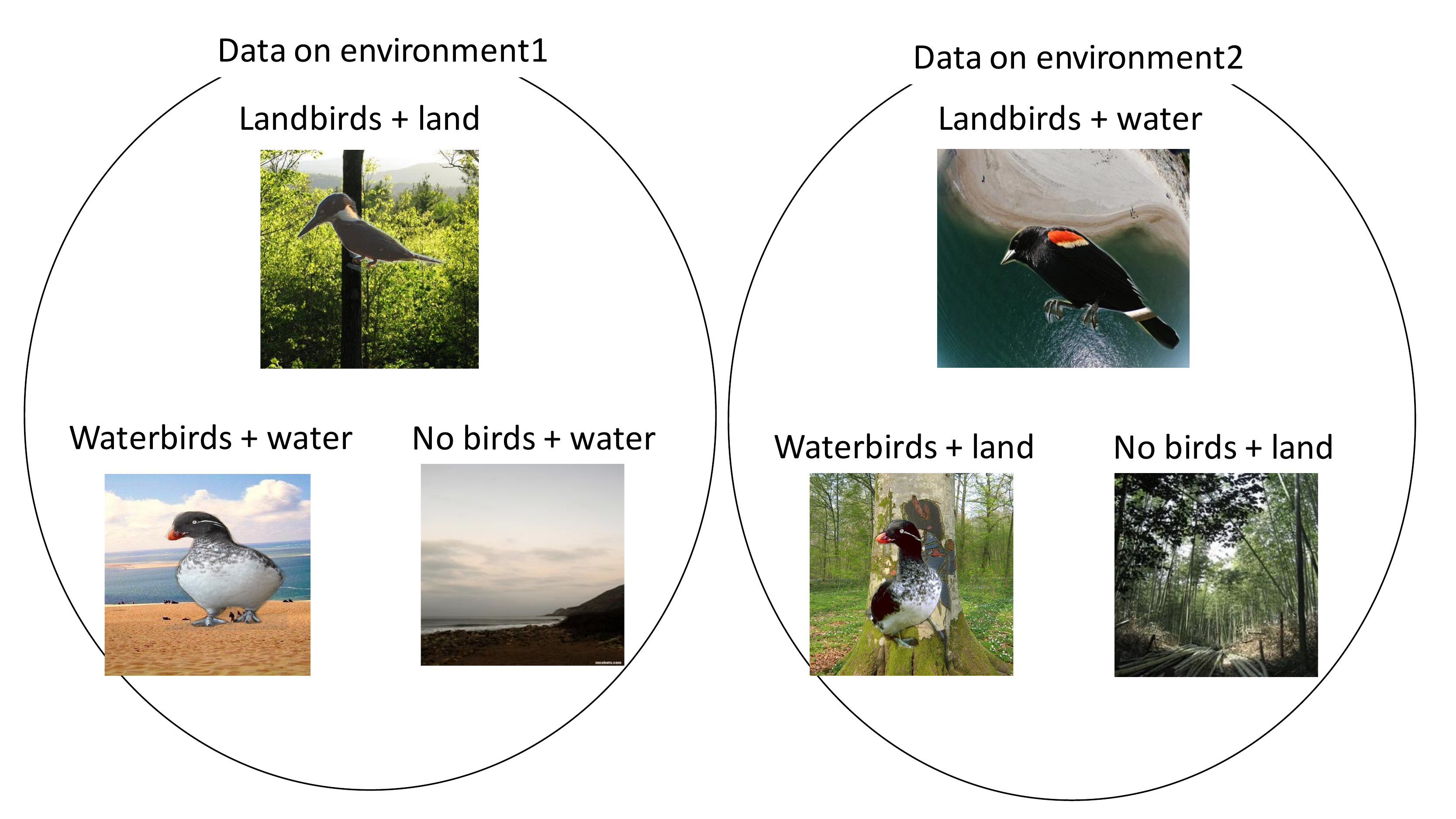}
 \vspace{-4mm}
  \vspace{2mm}
 \end{center}
\vskip -0.2in
\vspace{10mm}
\end{figure*}

\begin{table}[t]
\vspace{-6mm}
\vskip 0.2in
\begin{center}
\caption{Average Test Accuracies and SEs of Bird recognition problem (5 runs). {\em Oracle} shows a result of ERM with samples from both $e_1$ and $e_2$ given. TDV selects $\lambda$ which yields the highest performance on $e_2$. Best scores are \textbf{bolded}.}
 \vspace{3mm}
 \renewcommand{\arraystretch}{1.2}
 \begin{tabular*}{8.5cm}{@{\extracolsep{\fill}}|c|c|c|} \hline
     & \small Test Acc. on $e_1$ &   \small Test Acc. on $e_2$  \\ \hline \hline
   
  \small   Oracle&  \multicolumn{2}{c|}{ \small .875  (.018) } \\ \hline
     \small ERM&\small .902 (.008) &\small .317 (.044) \\
    \small FT&\small .909 (.012) &\small .364 (.028) \\
   \small FE&\small .767 (.024) &\small .052 (.013) \\ \hline \hline
   \small Ours +Our CV I&\small .897 (.020)  &\small \textbf{ .727 (.062) }\\ 
    \small Ours +Our CV II&\small .897 (.020)  &\small \textbf{ .727 (.062) }\\  \hline  \hline
    \small Ours +Tr-CV&\small .919 (.006)  &\small .651 (.031)\\ 
     \small Ours +LOD CV&\small  .338 (.048) &\small  .334 (.029) \\ \hline \hline
    \small Ours +TDV&\small .886 (.035)  &\small .782 (.020)\\\hline
  \end{tabular*}

  \vspace{2mm}
 \end{center}
\vskip -0.2in
\end{table}

\vspace{-2mm}
\paragraph{Bird recognition} Our method is applied to the Bird recognition problem \cite{S.Sagawa}, which aims to predict three labels $Y^e$ of images $X^e$: $waterbird$ ($Y^e$= 0), $landbird$  ($Y^e$= 1) and  $no$ $bird$  ($Y^e$= 2). The dataset is made by combining background images from the Places dataset \cite{Zhou} and bird images from the CUB dataset \cite{Wah}  in two different ways $\E := \{e_1, e_2\}$.  
\kf{In domain} $e_1$, we prepare three types of image: landbird image with land background, waterbird image  with water background, and no bird with land background (Figure \ref{fig:birds-vis}, left).  In 
\kf{domain} $e_2$, we have landbird images with water background, waterbird images with land background, and no bird with water background (Figure \ref{fig:birds-vis}, right).  For the sample of the target task, we used the \kf{domain} $e^*=e_1$ and generated $n^{e^*} =8649$ data $\D^{e^*} \sim P_{X^{e_1}, Y^{e_1}}$.  The sample in the higher level $\D^e_{ad}$ of $( X^e, Z^e )$, whose label is {\em landbird} ($Z^e= 0$) and  {\em no landbird} ($Z^e= 1$) ($i.e.$, $g(1)=0$ and $g(0)=g(2)=1$), is drawn from both $e_1$ and $e_2$ with $n^{e_1}=n^{e_2}=8649$. Here, we use $\D^{e^*}$ as $\D^{e_1}_{ad}$ with labels of $\D^{e^*}$ re-annotated by $g$. We made a predictor of $Y^e$ based on $X^e$, and evaluated the test accuracy in the two 
\kf{domains} $e=e_1, e_2$. We model $\Phi$ by ResNet50 \cite{K. He}. Setting the maximum epoch $5$, we select $(t, \lambda_{after})$ from $5 \times 5$ candidates with $t \in [5], \lambda_{after} \in \{10^0, 10^1,...,10^4\}$ by each CV method.  Table \ref{fig:birds-vis} shows test accuracies with $2162$ random samples for $e_1$ and $e_2$.  We can see that the proposed framework together with CV methods succeeded in capturing the predictor invariant to the change of background, while the other methods failed. \kf{ERM and FT show much higher accuracy for $e_1$ than Oracle and worst results for $e_2$, which implies that these methods learn spurious correlation in $\D^{e^*}$.}  


\section{Conclusion}\label{sec.CR}

\kf{
We have proposed a new framework of invariance learning --  assuming the availability of datasets for another task in higher label hierarchy, we obtain an invariant predictor for the target classification task using training data in a {\em single} domain with the help of the multiple data sets for the higher task.  This framework mitigates the difficulty of annotating many labels for the target task.  Additionally, we have proposed two CV methods for hyperparameter selection, which has been an outstanding problem of previous methods for invariant learning. Theoretical analysis has revealed that our methods select hyperparameter correctly under some settings.  The experimental results on synthesized and object recognition tasks have demonstrated the effectiveness of the proposed framework and CV methods.
}

\section*{Acknowledgements}
The research is supported by the Research Fellowships of Japan Society for the Promotion of Science for Young Scientists (Project number: 20J21396) and  JST CREST (Project number: JPMJCR2015).

\newpage
\onecolumn
\allowdisplaybreaks[1]
\appendix

\section{Proof of Theorem \ref{thm:HIT-CV}}
\label{append.proof_diff}

\begin{align}
\R^{(X,Y)}( p_{\theta} \circ \Phi ) -  \R^{(X,g(Y))}( p_{\theta} \circ \Phi  )   &= \int -  \log p_{\theta}(Y  | \Phi(X))  dP_{Y , \Phi(X)}  + \int     \log p_{\theta}(g(Y) | \Phi(X))   dP_{g(Y) ,  \Phi(X)} \notag \\
&= - \int  \log   \frac{ p_{\theta}(Y  | \Phi(X)) }{p_{\theta}(g(Y) | \Phi(X)) } dP_{(Y, \Phi(X))}  \notag\\   
&= - \int dP_{g(Y)} \int  \log   \frac{ p_{\theta}(Y | \Phi(X)) }{p_{\theta}(g(Y) | \Phi(X)) } dP_{(Y, \Phi(X)) | g(Y) }   \label{eq:ThmHIT-CV1}  \\   \notag
\end{align}
By the definition of  $p_{\theta}(y | \Phi(x), 
Y\in g^{-1}(z)))$ in Theorem \ref{thm:HIT-CV},  $\frac{ p_{\theta}(y | \Phi(x)) }{p_{\theta}(g(y) | \Phi(x)) }  =  p_{\theta}(y  | \Phi(x), Y \in  g^{-1}(z))$ holds, where $z=g(y)$. Therefore, we obtain 

\begin{align}
\mathrm{(\ref{eq:ThmHIT-CV1})}& = - \int dP_{g(Y)} \int  \log   \frac{ p_{\theta}(Y  | \Phi(X)) }{p_{\theta}(g(Y) | \Phi(X)) } dP_{(Y, \Phi(X)) | g(Y) }     \notag \\ 
&= - \int dP_{g(Y)} \int  \log  p_{\theta}(Y  | \Phi(X), Y=g(Y) ) dP_{(Y, \Phi(X)) | g(Y) } 
  \label{eq:ThmHIT-CV2}  \\   
&=   - \sum_{z \in \Z } P(g(Y) = z)  \int  \log  p_{\theta}(Y  | \Phi(X), 
Y\in g^{-1}(z) ) dP_{(Y, \Phi(X)) | g(Y)=z } \notag  \\
&=   - \sum_{z^{\scalebox{0.5}{$\cancel{\hookrightarrow}$}} \in \Z^{\cancel{\hookrightarrow}} } P(g(Y) = z^{\scalebox{0.5}{$\cancel{\hookrightarrow}$}})  \int  \log  p_{\theta}(Y  | \Phi(X), 
Y\in g^{-1}(z^{\scalebox{0.5}{$\cancel{\hookrightarrow}$}}) ) dP_{(Y, \Phi(X)) | g(Y)=z^{\scalebox{0.5}{$\cancel{\hookrightarrow}$}} } \notag  \\
&~~~~~~~~~~~~~~~~~~~~~~~~~~~~~~~~~~~~~+  \sum_{z^{\scalebox{0.5}{$\hookrightarrow$}} \notin \Z^{\cancel{\hookrightarrow}} } P(g(Y) = z^{\scalebox{0.5}{$\hookrightarrow$}})  \int  \log  p_{\theta}(Y  | \Phi(X), 
Y\in g^{-1}(z^{\scalebox{0.5}{$\hookrightarrow$}}) ) dP_{(Y, \Phi(X)) | g(Y)=z^{\scalebox{0.5}{$\hookrightarrow$}} }. \label{eq:ThmHIT-CV3} \\   \notag
\end{align}
Noting that, for $ \forall z^{\scalebox{0.5}{$\hookrightarrow$}} \notin \Z^{\cancel{\hookrightarrow}}$ and $y :=g^{-1} (z^{\scalebox{0.5}{$\hookrightarrow$}} )$\footnote{$z^{\scalebox{0.5}{$\hookrightarrow$}} \notin \Z^{\cancel{\hookrightarrow}}$ implies that $| g^{-1} (z^{\scalebox{0.5}{$\hookrightarrow$}}) | =1 $ and therefore, $g^{-1} (z^{\scalebox{0.5}{$\hookrightarrow$}} )$ is determined uniquely. Note that there is no chance that $| g^{-1} (z^{\scalebox{0.5}{$\hookrightarrow$}}) | =0 $ by the surjectivity of $g$.},  $p_{\theta}(y| \Phi(x), Y \in g^{-1}(z^{\scalebox{0.5}{$\hookrightarrow$}})))=1$ holds, we can see that  $\log p_{\theta}(y, | \Phi(x), Y=g^{-1}(z^{\scalebox{0.5}{$\hookrightarrow$}}))) = 0$. It leads us to  the equality  $\sum_{z^{\scalebox{0.5}{$\hookrightarrow$}} \notin \Z^{\cancel{\hookrightarrow}} } P(g(Y) = z^{\scalebox{0.5}{$\hookrightarrow$}})  \int  \log  p_{\theta}(Y  | \Phi(X), Y \in g(z^{\scalebox{0.5}{$\hookrightarrow$}}) ) dP_{(Y, \Phi(X)) | g(Y)=z^{\scalebox{0.5}{$\hookrightarrow$}} } = 0$, which concludes the proof. $\Box$

 \section{Proof of Theorem \ref{Thm:justfyHITCV}}\label{append:CV-theo}
 We restate Theorem \ref{Thm:justfyHITCV} with some notation arrangements.
   \begin{theo}\label{Thm:justfyHITCV1:re}
Let $\X:= \X_1 \times\X_2$ where $\X_1 := \bR^{n_1}$ and $\X_2 := \bR^{n_2}$ with $n_1,n_2 \in \bN$. For any random variable $X$  on $\X$, $X_1$ and $X_2$ denote its $\X_1$- and $\X_2$-component of $X$, respectively. Assume that, for  $\lambda \in \Lambda$, there corresponds a projection $\Phi^{\lambda}: \X \rightarrow \bR^{n_{\lambda}}$ ($n_{\lambda} \leq n_1+ n_2$). $\Phi_1^{\lambda}$ and  $\Phi_2^{\lambda}$ denote  an $\X_1$ and $\X_2$ components of $\Phi^{\lambda}$. For $i =1,2$,  if $\Phi^{\lambda}$ does not have  an $\X_i$-component, we denote $\mathrm{Im} \Phi^{\lambda}_i = \emptyset$.  $\Phi(X)$, $\Phi^{\lambda}_1(X)$ and $\Phi^{\lambda}_2(X)$ are abbreviated by $X^{\lambda}$, $X_1^{\lambda}$ and $X_2^{\lambda}$, respectively. Fixing a random variable $(X^I_1, Y^I)$ on $\X_1 \times \Y$, set 
$$T_{all} := \left\{ (X, Y): ~random~variable ~on ~ \X \times \Y \left| P_{Y,X_1} = P_{Y^I, X^I_1} \right. \right\}.$$
Fix $(X^*, Y^* ) \in T_{all}$ and $T_{ad} \subset T_{all}$. For $\lambda \in \Lambda$, $p^{*, \lambda}: \bR^{n_{\lambda}} \rightarrow \P_{\Y}$ denotes the conditional probability density function of $P(Y^*| \Phi^{\lambda}(X^*))= P(Y^* | X^{*, \lambda})$. For $\lambda \in \Lambda$, define $\R^{o.o.d.}(\lambda)$ and $\R^{I}(\lambda)$ by
\begin{align*}
&\R^{o.o.d.}(\lambda) := \max_{(X,Y) \in T_{all}} \R^{(X,Y)}(p^{*, \lambda}  \circ \Phi^{\lambda}),\\
&\R^{I}(\lambda) := \max \Biggl\{   \max_{(X,Y) \in T_{ad}}   \R^{(X,g(Y))}(p^{*, \lambda}  \circ \Phi^{\lambda}) ,  \R^{(X^*,Y^*)}(p^{*, \lambda}  \circ \Phi^{\lambda}) \Biggr\}
\end{align*}\label{eq.HOOD}
respectively. Assume that the following two conditions hold:
\vspace{-3mm}
\begin{itemize}
\item[(I)]  $\exists \lambda^{I} \in \Lambda$ s.t. $\Phi^{\lambda^{I}}=  \Phi^{\X_1}$.
\item[(II)]  For sufficiently small $\e \ll 1$, the following statement holds:\\
$\forall \lambda$ with $\mathrm{Im} \Phi^{\lambda}_2 \neq \emptyset$, there exists $(X^{e_\lambda}, Y^{e_\lambda})$ such that 
 $P(g(Y^*)   | X^{* , \lambda}  ) \leq e^{-\beta} -  \e$ holds $P_{X^{e_\lambda}, Y^{e_\lambda}}$-almost everywhere.
\end{itemize}
Then, $\argmin_{\lambda \in \Lambda} \R^{I}(\lambda) \subset \argmin_{\lambda \in \Lambda} \R^{o.o.d.}(\lambda)$ holds.
\end{theo}

$T_{all}$, $T_{ad}$ and $(X^*, Y^*)$ correspond to $\{(X^e, Y^e) \}_{e \in \E}$,$\{(X^e, Y^e) \}_{e \in \E_{ad}}$ and $(X^{e^*}, Y^{e^*})$ in Theorem \ref{Thm:justfyHITCV} respectively. (a) and (b) in Theorem \ref{Thm:justfyHITCV} are represented by the  construction of $T_{all}$. (c) and (d) in Theorem \ref{Thm:justfyHITCV} are represented by (I) and (II) respectively.


To prove Theorem  \ref{Thm:justfyHITCV1:re}, we prepare three lemmas. In the lemmas, notations are same as in Theorem  \ref{Thm:justfyHITCV1:re} and condition (I) and (II) in Theorem \ref{Thm:justfyHITCV1:re} are also imposed on.

 \begin{lemm}\label{Lem:min of I}
  $ \lambda^{I} \in \argmin_{\lambda \in \Lambda} \R^{o.o.d.} (\lambda)$. 
 \end{lemm} 
 
 \begin{lemm}\label{Lem:min is no  II}
 Assume that $\hat{\lambda} \in \argmin_{\lambda \in \Lambda} \R^{I}(\lambda) $. Then  $\mathrm{Im} \Phi_2^{\hat{\lambda}} =  \emptyset$.
 \end{lemm}
 
 \begin{lemm}\label{Lem:no  II is H=0}
 If  $\hat{\lambda} \in \Lambda$ satisfies $\mathrm{Im} \Phi_2^{\hat{\lambda}} =  \emptyset$,   $\R^{I}(\hat{\lambda})  = \R^{o.o.d.}(\hat{\lambda})  $. 
  \end{lemm}
 Before proving the above lemmas, we prove Theorem \ref{Thm:justfyHITCV1:re} suppose that they hold.\\
 {\bf{proof of Theorem \ref{Thm:justfyHITCV1:re}}}.\\
 Take $\hat{\lambda} \in \argmin \R^{I} (\lambda)$. Then, $\mathrm{Im} \Phi_2^{\hat{\lambda}}=  \emptyset  $ holds by Lemma \ref{Lem:min is no  II} and therefore, $ \R^{I} (\hat{\lambda}) = \R^{o.o.d.} (\hat{\lambda}) $ holds by Lemma \ref{Lem:no  II is H=0}. Moreover, $\R^{o.o.d.} (\hat{\lambda})  \geq  \R^{o.o.d.} (\lambda^{I}) $ holds by Lemma \ref{Lem:min of I} and  $\R^{o.o.d.} (\lambda^{I}) = \R^{I} (\lambda^{I})$ holds by \ref{Lem:no  II is H=0}.\footnote{Note that, since $\Phi^{\lambda^I}$ is the projection onto $\X_1$, $\mathrm{Im} \Phi_2^{\lambda^I} =  \emptyset $.} By the assumption $\hat{\lambda} \in \argmin_{\lambda \in \Lambda} \R^{I}(\hat{\lambda}) $, $  \R^{I} (\lambda^{I}) \geq \R^{I} (\hat{\lambda})  $ holds. Arranging these inequalities, we obtain 
 \begin{equation}\label{prove:ourCV}
 \R^{I} (\hat{\lambda}) = \R^{o.o.d.} (\hat{\lambda})  \geq  \R^{o.o.d.} (\lambda^{I}) = \R^{I} (\lambda^{I}) \geq \R^{I} (\hat{\lambda}). 
 \end{equation}
 Since  the left and right ends of  (\ref{prove:ourCV}) are connected by the same value  $\R^{I} (\hat{\lambda})$, the inequalities in (\ref{prove:ourCV}) must be equalities.  Hence, we obtain the equality $ \R^{o.o.d.} (\hat{\lambda})  =  \R^{o.o.d.}(\lambda^{I})$.  By the minimality of $\lambda^I$ (Lemma \ref{Lem:min of I}), the equality $ \R^{o.o.d.} (\hat{\lambda})  =  \R^{o.o.d.}(\lambda^{I})$ implies $\hat{\lambda} \in \argmin \R^{o.o.d.}(\lambda)$, which concludes the proof. $\Box$\\
{\bf{proof of Lemma \ref{Lem:min of I}} }\\
 It suffices to prove that, for any $\hat{\lambda} \in \Lambda$ and  $(\bar{X}, \bar{Y}) \in T_{all}$, there exists $(\bar{\bar{X}}, \bar{\bar{Y}}) \in T_{all}$ such that $\int - \log p^{*, \hat{\lambda}}(\bar{\bar{Y}} | \Phi^{\hat{\lambda}}(\bar{\bar{X}})) d P_{\bar{\bar{X}},\bar{\bar{Y}}} \geq \int - \log p^{*, \lambda^I} (\bar{Y} | \Phi^{\lambda^I}(\bar{X}) ) d P_{\bar{X},\bar{Y}}$.
 Take $(\bar{\bar{X}}, \bar{\bar{Y}}) \in T_{all}$ such that its distribution is $P_{X_1^I, Y^I} \times P_{\bar{X}_2}$. Here, $P_{\bar{X}_2}$ denotes a marginal distribution of $(\bar{X}, \bar{Y}) \in T_{all}$ on $\X_2$ and $P_{X_1^I, Y^I} \times P_{\bar{X}_2}$ denotes the product of $P_{X_1^I, Y^I}$ and $P_{\bar{X}_2}$.
 \begin{align}
 \int - \log p^{*,\hat{\lambda}} (\bar{\bar{Y}} | \Phi^{\hat{\lambda}}(\bar{\bar{X}})) d P_{\bar{\bar{X}}, \bar{\bar{Y}}} &=  \int - \log p^{*, \hat{\lambda}} (\bar{\bar{Y}} | \bar{\bar{X}}^{\hat{\lambda}}_1,  \bar{\bar{X}}^{\hat{\lambda}}_2)  d P_{\bar{\bar{X}}, \bar{\bar{Y}}} \notag \\
 &=  \int - \log p^{*, \hat{\lambda}} (Y^I | X^{I, \hat{\lambda}}_1,  \bar{X}^{\hat{\lambda}}_2) d  (P_{X_1^I, Y^I} \times P_{\bar{X}_2}) \notag \\
  &= \int dP_{\bar{X}_2} \int - \log p^{*, \hat{\lambda}} (Y^I | X^{I, \hat{\lambda}}_1,  \bar{X}^{\hat{\lambda}}_2) d  P_{X_1^I, Y^I}.  \label{eq:Lem:min of I} \\ \notag
  \end{align}
 Note that, for $\forall x \in  \mathrm{Im} \Phi_2^{\hat{\lambda}} $, $\int - \log p^{*,\hat{\lambda}}(Y^I | X^{I, \hat{\lambda}}_1,  \bar{X}^{\hat{\lambda}}_2= 
 x^{\hat{\lambda}}_2)) d  P_{X_1^I, Y^I} \geq    \int  - \log p^{*, \lambda^I} (Y^I | X^I_1) d P_{X^I_1, Y^I} $ holds since a minimum of the cross entropy loss is attained if and only if $p_{\theta}$ corresponds to the  conditional distribution function  $p^{*, \lambda^I}$ of $P_{Y^I | X_1^I}$.\footnote{By the construction of $T_{all}$, $P_{Y| X_1}$ corresponds to $P_{Y^I| X_1^I}$ for any $(X, Y) \in T_{all}$. Therefore, $P_{Y^I| X_1^I} = P_{Y^* | X^*_1} =  P_{Y^* | \Phi^{\lambda^I}(X^*)}$ holds, which implies that the conditional probability density function of $P_{Y^I | X_1^I}$ is $p^{*, \lambda^I}$.} Therefore, we can see that 
 \begin{align*} 
 \mathrm{(\ref{eq:Lem:min of I})} & \geq  \int dP_{\bar{X}_2}  \int  - \log p^{*, \lambda^I}(Y^I | X^I_1) d P_{X^I_1, Y^I}  \\
  &  =\int dP_{\bar{X}_2| X_1^I, Y^I}  \int  - \log p^{*, \lambda^I}(Y^I | X^I_1) d P_{X^I_1, Y^I}   \\
  &  =\int  \int  - \log p^{*, \lambda^I}(Y^I | X^I_1) d P_{X^I_1, Y^I} dP_{\bar{X}_2| X_1^I, Y^I}   \\
  &  =\int   - \log p^{*, \lambda^I} (\bar{Y} | \bar{X}_1) dP_{\bar{X}, \bar{Y}}   \\
   &  =\int   - \log p^{*, \lambda^I} (\bar{Y} |\Phi^{\lambda^I}( \bar{X})) dP_{\bar{X}, \bar{Y}} ,  \\
  \end{align*}
and therefore, it concludes the proof. Here, the first equality holds because $- \log p^{*, \lambda^I} (Y^I | X^I_1) $ is not affected by $\X_2$. $\Box$\\
 {\bf{proof of Lemma \ref{Lem:min is no  II}}}.\\
 Let us prove the contraposition of Lemma \ref{Lem:min is no  II}. Take $\hat{\lambda} \in \Lambda$ with $\mathrm{Im} \Phi_2^{\hat{\lambda} } \neq  \emptyset$. To prove that $\hat{\lambda} \notin \argmin \R^{I}(\lambda)$, we may   prove that $\R^{I}(\hat{\lambda}) > \R^{I} (\lambda^I)$ since $\lambda^I \in \Lambda$ (Assumption (I) in the statement). To show this, it suffices to prove the following statement:
\begin{align}
  \exists (\bar{X}, \bar{Y}) \in T_{ad}~~ s.t.~~\int -\log p^{*, \hat{\lambda}}  (g(\bar{Y}) | \bar{X}^{\hat{\lambda}}) dP_{\bar{X},g(\bar{Y})}  > \R^{I} (\lambda^I).\label{eq:Lem:min is no  II}
 \end{align}
From Condition (II), we can take $( X^{e_{\hat{\lambda}}},Y^{e_{\hat{\lambda}}}) \in T_{ad}$ such that the following statement holds:
 \begin{center}
\rm{ $P(g(Y^*)   | X^{* , \hat{\lambda}}   ) \leq e^{-\beta} -  \e$ holds $P_{X^{e_{\hat{\lambda}}}, Y^{e_{\hat{\lambda}}}}$-almost everywhere.}
 \end{center}
Before proving (\ref{eq:Lem:min is no  II}), we prepare one  supplementary inequality:
\\   {\bf{Supplementary Inequality}}\\
$$\int -\log p^{*,\hat{\lambda}}  (g(Y^{e_{\hat{\lambda}}}) | X^{e_{\hat{\lambda}},\hat{\lambda} }) dP_{X^{e_{\hat{\lambda}}},g(Y^{e_{\hat{\lambda}}})}  \geq -\log \left\{ {e^{ - \beta }- \epsilon } \right\}.$$ 

To prove the inequality, note that $P_{X^{e_{\hat{\lambda}}},g(Y^{e_{\hat{\lambda}}})} (A) = 0$, where 
$$A:= \left\{ \bigl(  x, y \bigr)  \in \X \times \Y \left|  P(g(Y^*)  = g(y)   | X^{* , \hat{\lambda}}  = \Phi^{\hat{\lambda}}(x)  ) > e^{-\beta} -  \e \right. \right\}$$
since“ $P(g(Y^*)   | X^{* , \hat{\lambda}}   ) \leq e^{-\beta} -  \e$ holds $P_{X^{e_{\hat{\lambda}}}, Y^{e_{\hat{\lambda}}}}$-almost everywhere"   holds.

Therefore, we can see that $$\int -\log p^{*,\hat{\lambda}}  (g(Y^{e_{\hat{\lambda}}}) | X^{e_{\hat{\lambda}}, \hat{\lambda}}) dP_{X^{e_{\hat{\lambda}}},g(Y^{e_{\hat{\lambda}}})}   =\int_{\X \times \Y - A}  -\log p^{*,\hat{\lambda}}  (g(Y^{e_{\hat{\lambda}}}) | X^{e_{\hat{\lambda}}, \hat{\lambda}}) dP_{X^{e_{\hat{\lambda}}},g(Y^{e_{\hat{\lambda}}})}   $$.  Note that for any $(x, y) \in \X \times \Y - A$, $ -\log p^{*,\hat{\lambda}}  (g(y) | x^{\hat{\lambda}}) \geq  e^{-\beta} -  \e $; indeed, since 

$$ \X \times \Y - A = \left\{ \bigl(  x, y \bigr)  \in \X \times \Y \left|  P(g(Y^*)  = g(y)   | X^{* , \hat{\lambda}}  = \Phi^{\hat{\lambda}}(x)  ) \leq e^{-\beta} -  \e \right. \right\},$$ for any $(x, y) \in \X \times \Y - A$,  $$ -\log p^{*,\hat{\lambda}}  (g(y) | x^{\hat{\lambda}}) = - \log  P(g(Y^*)  = g(y)   | X^{* , \hat{\lambda}}  = \Phi^{\hat{\lambda}}(x)  )   \geq   -\log \left\{ {e^{ - \beta }- \epsilon } \right\}.  $$
Then, we obtain 
\begin{align*}
\int -\log p^{*,\hat{\lambda}}  (g(Y^{e_{\hat{\lambda}}}) | X^{e_{\hat{\lambda}}, \hat{\lambda}}) dP_{X^{e_{\hat{\lambda}}},g(Y^{e_{\hat{\lambda}}})}   &=\int_{\X \times \Y - A}  -\log p^{*,\hat{\lambda}}  (g(Y^{e_{\hat{\lambda}}}) | X^{e_{\hat{\lambda}}, \hat{\lambda}}) dP_{X^{e_{\hat{\lambda}}},g(Y^{e_{\hat{\lambda}}})}    \\
 &\geq   -\log \left\{ {e^{ - \beta }- \epsilon } \right\}.
 \end{align*}
 {\bf{Proof of Inequality (\ref{eq:Lem:min is no  II})}}.\\
 \begin{align}
& \int -\log p^{*,\hat{\lambda}}  (g(Y^{e_{\hat{\lambda}}}) | X^{e_{\hat{\lambda}}, \hat{\lambda}}) dP_{X^{e_{\hat{\lambda}}},g(Y^{e_{\hat{\lambda}}})}     \geq -\log \left\{ {e^{ - \beta }- \epsilon } \right\}  >  -\log \left\{ {e^{ - \beta }} \right\}  
 =  \beta
 =  H(Y^* | X_1^*) 
  \end{align}
 Note that  $H(Y^* | X_1^*) = \R^{o.o.d.}(\lambda^I)$; indeed, 
 \begin{align}
  \R^{o.o.d.} (\lambda^I) &= \max_{(X, Y) \in T_{all}} \R^{(X, Y) } (p^{*, \lambda^I}\circ \Phi^{\lambda^I}) = \max_{(X, Y) \in T_{all}} \int - \log p^{*, \lambda^I} (Y | \Phi^{\lambda^I}(X)) dP_{X, Y} \notag \\
  &=  \max_{(X, Y) \in T_{all}} \int - \log p^{*, \lambda^I} (Y | \Phi^{\lambda^I}(X)) dP_{\Phi^{\lambda^I}(X), Y} \notag\\
  &=  \max_{(X, Y) \in T_{all}} \int - \log p^{*, \lambda^I} (Y |X_1) dP_{X_1, Y}. \label{eq:Lem:min is no  II:3 }
 \end{align}
 Noting that  $P_{X_1, Y}  =P_{X_1^*, Y^*} (=  P_{X_1^I, Y^I})$ for any $(X, Y) \in T
 _{all}$ and   $p^{*, \lambda^I}$ coincides with the conditional probability density function of $P_{Y^* | X_1^*}$, we can see that
 $$ \mathrm{(\ref{eq:Lem:min is no  II:3 })}=   \int - \log p^{*, \lambda^I} (Y^* |X_1^*) dP_{X_1^*, Y^*}  = H (Y^*| X_1^*).$$
 
 Hence, we can derive $\int -\log p^{*,\hat{\lambda}}  (g(Y^{e^\lambda}) | X^{e_{\hat{\lambda}}, \hat{\lambda}}) dP_{X^{e_{\hat{\lambda}}},g(Y^{e_{\hat{\lambda}}})} > H (Y^*| X_1^*) = \R^{o.o.d.}(\lambda^I)$, which concludes the proof. $\Box$
  \paragraph{{\bf{proof of Lemma \ref{Lem:no  II is H=0}}}}.\\
Take $\hat{\lambda} \in \Lambda$ that satisfies $\mathrm{Im} \Phi_2^{\hat{\lambda}} =  \emptyset$.
Then, $P_{\Phi(X^{\hat{\lambda}}),Y} = P_{\Phi(X^{I,\hat{\lambda}}), Y^I}$ holds for $\forall (X, Y) \in T_{all}$ because of $P_{X_1,Y}=P_{X_1^I,Y}$, 
and therefore, $\R^{(X,g(Y))}(p^{*, \hat{\lambda}}  \circ \Phi^{\hat{\lambda}}) = \R^{(X^I,g(Y^I))}(p^{*, \hat{\lambda}}  \circ \Phi^{\hat{\lambda}}) $ and $\R^{(X^*,Y^*)}(p^{*, \hat{\lambda} } \circ \Phi^{\hat{\lambda}}) = \R^{(X^I,Y^I)}(p^{*, \hat{\lambda}}  \circ \Phi^{\hat{\lambda}})$ hold.  These two equalities lead the following equality:
    \begin{align}
      \R^{I} (\hat{\lambda}) &= \max \Biggl\{   \max_{(X,Y) \in T_{ad}}   \R^{(X,g(Y))}(p^{*, \hat{\lambda} } \circ \Phi^{\hat{\lambda}}) ,  \R^{(X^*,Y^*)}(p^{*, \hat{\lambda}}  \circ \Phi^{\hat{\lambda}}) \Biggr\}  \notag \\
      &= \max \Biggl\{    \R^{(X^I,g(Y^I))}(p^{*, \hat{\lambda}}  \circ \Phi^{\hat{\lambda}}) ,  \R^{(X^I,Y^I)}(p^{*, \hat{\lambda}  }\circ \Phi^{\hat{\lambda}}) \Biggr\}    \label{eq_aprrox1}
      \end{align}
      By Theorem  \ref{thm:HIT-CV}, 
      \begin{align*}
      &R^{(X^I,Y^I)}(p^{*, \hat{\lambda}}  \circ \Phi^{\hat{\lambda}})   \\
      &= \R^{(X^I,g(Y^I))}(p^{*, \hat{\lambda}}  \circ \Phi^{\hat{\lambda}}) +\sum_{z^{\scalebox{0.5}{$\cancel{\hookrightarrow}$}} \in \Z^{\cancel{\hookrightarrow}}}P(Y^I = g^{-1}(z^{\scalebox{0.5}{$\cancel{\hookrightarrow}$}})) \int -\log p^{*, \hat{\lambda}}   (Y^I  | \Phi(X^I), Y^I \in g^{-1}(z^{\scalebox{0.5}{$\cancel{\hookrightarrow}$}})) dP_{X^I,Y^I | Y^I \in  g^{-1}(z^{\scalebox{0.5}{$\cancel{\hookrightarrow}$}})}  \\
      & \geq  \R^{(X^I,g(Y^I))}(p^{*,\hat{ \lambda} } \circ \Phi^{\hat{\lambda}}) 
      \end{align*}
holds and therefore, $(\ref{eq_aprrox1}) =  \R^{(X^I,Y^I)}(p^{*, \hat{\lambda} } \circ \Phi^{\hat{\lambda}})$. Since $P_{\Phi^{\hat{\lambda}}(X), Y}$ are the same for  $T_{all}$, 
\begin{align*}
     \R^{(X^I,Y^I)}(p^{*, \hat{\lambda}}  \circ \Phi^{\lambda}) = \max_{(X,Y) \in T_{all}} \R^{(X,Y)}(p^{*, \hat{\lambda}}  \circ \Phi^{\hat{\lambda}})  = \R^{o.o.d.}(p^{*, \hat{\lambda} } \circ \Phi^{\hat{\lambda}}),
   \end{align*}
which concludes the proof.
    

 \section{Proof of Theorem \ref{Thm:justfyHITCV2}}
 \label{append.proof_cv2}
\begin{theo}\label{Thm:justfyHITCV2:re}
Notations are same as in the statement of Theorem \ref{Thm:justfyHITCV1:re}. Define $\R^{II}(\lambda)$ by 
\begin{align*}
&\R^{II}(\lambda) :=  \max_{(X,Y) \in T_{ad} \cup \{(X^*, Y^*) \}} \Biggl\{   \R^{(X,g(Y))}(p^{*, \lambda}  \circ \Phi^{\lambda})+ \sum_{z^{\scalebox{0.5}{$\cancel{\hookrightarrow}$}} \in \Z^{\cancel{\hookrightarrow}}} \Bigl\{ P(g(Y)= z^{\scalebox{0.5}{$\cancel{\hookrightarrow}$}} )   \\
&~~~~~~~~~~~~~~~~~~~~~~~~~~~~~~~~~~~~~~~~~~~~~~~~~~~~~~~~~\cdot  \int-\log p^{*, \lambda} (Y^* | X^{*,{\lambda}}, Y^* \in  g^{-1}(z^{\scalebox{0.5}{$\cancel{\hookrightarrow}$}})) dP_{(X^*,Y^*) | Y^* \in g^{-1}(^{\scalebox{0.5}{$\cancel{\hookrightarrow}$}})} \Bigl\} \Biggl\}.
\end{align*}\label{eq.HOOD}
In addition to the condition (I), the following condition (II)' hold: 
\vspace{-3mm}
\begin{itemize}
\item[(II)']  For a sufficiently small $\e \ll 1$, the following statement holds:\\
$\forall \lambda \in \Lambda$ with $\mathrm{Im} \Phi^{\lambda}_2 \neq \emptyset$, there exists $(X^{e_\lambda}, Y^{e_\lambda})$ such that 
 $P(g(Y^*)   | X^{* , \lambda}  ) \leq e^{-\beta_{\lambda}} -  \e$ holds $P_{X^{e_\lambda}, Y^{e_\lambda}}$-almost everywhere. Here, 
\begin{align*} 
&\beta_{\lambda} := H(Y^* | X_1^*) - \sum_{z^{\scalebox{0.5}{$\cancel{\hookrightarrow}$}}  \in \Z^{\cancel{\hookrightarrow}}}\Bigl\{ P(g(Y^*) = z^{\scalebox{0.5}{$\cancel{\hookrightarrow}$}} ) \cdot \int - \log p^{*, \lambda}  (Y^* | X^{*, \lambda}, Y^* \in g^{-1}(z^{\scalebox{0.5}{$\cancel{\hookrightarrow}$}} ) )dP_{(X^*,Y^*) | Y^* \in g^{-1}(z^{\scalebox{0.5}{$\cancel{\hookrightarrow}$}} )} \Bigl\}.
\end{align*}
\end{itemize}
Then, $\argmin_{\lambda \in \Lambda} \R^{II}(\lambda) \subset \argmin_{\lambda \in \Lambda} \R^{o.o.d.}(\lambda)$.
\end{theo}

\begin{lemm}\label{Lem:min is no  II_2}
 Assume that $\hat{\lambda} \in \argmin_{\lambda \in \Lambda} \R^{II}(\lambda) $. Then  $\mathrm{Im} \Phi_2^{\hat{\lambda}} =  \emptyset$.
 \end{lemm}
 
\begin{lemm}\label{Lem:no  II is H=0_2}
 If  $\hat{\lambda} \in \Lambda$ satisfies $\mathrm{Im} \Phi_2^{\hat{\lambda}} =  \emptyset$,   $\R^{II}(\hat{\lambda})  = \R^{o.o.d.}(\hat{\lambda})  $. 
\end{lemm}
\hspace{-7mm}{ \bf{proof of Theorem \ref{Thm:justfyHITCV2:re}} }\\
 Combining the above two lemmas and Lemma \ref{Lem:min of I}, we can show the desired statement essentially the same as the one in the proof of Theorem \ref{Thm:justfyHITCV2:re}.\\
 {\bf{proof of Lemma \ref{Lem:min is no  II_2}}}.\\
 Let us prove the contraposition of Lemma \ref{Lem:min is no  II}. Take $\hat{\lambda} \in \Lambda$ with $\mathrm{Im} \Phi_2^{\hat{\lambda} } \neq  \emptyset$. To prove that $\hat{\lambda} \notin \argmin \R^{II}(\lambda)$, we may   prove that $\R^{II}(\hat{\lambda}) > \R^{II} (\lambda^I)$ since $\lambda^I \in \Lambda$ (Assumption (I) in the statement). To show this, it suffices to prove the following statement:
\begin{align}
  &\exists (\bar{X}, \bar{Y}) \in T_{ad} ~~ s.t.\notag \\
 &\int -\log p^{*, \hat{\lambda}}  (g(\bar{Y}) | \bar{X}^{\hat{\lambda}}) dP_{\bar{X},g(\bar{Y})} + \sum_{z^{\scalebox{0.5}{$\cancel{\hookrightarrow}$}}\in \Z^{\cancel{\hookrightarrow}}} P(g(\bar{Y})= z^{\scalebox{0.5}{$\cancel{\hookrightarrow}$}} ) \cdot \int-\log p^{*, \hat{\lambda}}  (Y^* | &X^{*,\hat{\lambda}}, Y^* \in  g^{-1}(z^{\scalebox{0.5}{$\cancel{\hookrightarrow}$}})) &dP_{(X^*,Y^*) | Y^* \in  g^{-1}(z^{\scalebox{0.5}{$\cancel{\hookrightarrow}$}})} \notag \\
 && > \R^{II} (\lambda^I).\label{eq:Lem:min is no  II_2}
 \end{align}

 Take $( X^{e_{\hat{\lambda}}},Y^{e_{\hat{\lambda}}}) \in T_{ad}$ such that the following statement holds:
 \begin{center}
\rm{ $P(g(Y^*)   | X^{* , \hat{\lambda}}   ) \leq  e^{-\beta_{\hat{\lambda}}} -  \e$ holds $P_{X^{e_{\hat{\lambda}}}, Y^{e_{\hat{\lambda}}}}$-almost everywhere.}
 \end{center}
Before proving (\ref{eq:Lem:min is no  II_2}), we prepare one  supplementary inequality:
\\   {\bf{Supplementary Inequality}}\\
$$\int -\log p^{*,\hat{\lambda}}  (g(Y^{e_{\hat{\lambda}}}) | X^{e_{\hat{\lambda}},\hat{\lambda} }) dP_{X^{e_{\hat{\lambda}}},g(Y^{e_{\hat{\lambda}}})}  \geq -\log \left\{ {e^{ - \beta_{\hat{\lambda}} }- \epsilon } \right\}.$$ 

We can prove the  inequality same as in the proof of Lemma \ref{Lem:min is no  II}, and therefore,  omit the proof. \\
 {\bf{Proof of Inequality (\ref{eq:Lem:min is no  II_2})}}.\\
 \begin{align}
&\hspace{-10mm}\int -\log p^{*,\hat{\lambda}}  (g(Y^{e_{\hat{\lambda}}}) | X^{e_{\hat{\lambda}},\hat{\lambda} }) dP_{X^{e_{\hat{\lambda}}},g(Y^{e_{\hat{\lambda}}})}+ \sum_{z^{\scalebox{0.5}{$\cancel{\hookrightarrow}$}} \in \Z^{\cancel{\hookrightarrow}}} P(g({Y^{e_{\hat{\lambda}}}})= z^{\scalebox{0.5}{$\cancel{\hookrightarrow}$}} ) \cdot \int-\log p^{*, \hat{\lambda}}  (Y^* | X^{*,\hat{\lambda}}, Y^* \in  g^{-1}(z^{\scalebox{0.5}{$\cancel{\hookrightarrow}$}})) dP_{(X^*,Y^*) | Y^* \in  g^{-1}(z^{\scalebox{0.5}{$\cancel{\hookrightarrow}$}})} \notag\\
& \geq -\log \left\{ {e^{ - \beta_{\hat{\lambda}} }- \epsilon } \right\} + \sum_{z^{\scalebox{0.5}{$\cancel{\hookrightarrow}$}} \in \Z^{\cancel{\hookrightarrow}}} P(g({Y^{e_{\hat{\lambda}}}})= z^{\scalebox{0.5}{$\cancel{\hookrightarrow}$}} ) \cdot \int-\log p^{*, \hat{\lambda}}  (Y^* | X^{*,\hat{\lambda}}, Y^* \in  g^{-1}(z^{\scalebox{0.5}{$\cancel{\hookrightarrow}$}})) dP_{(X^*,Y^*) | Y^* \in g^{-1}(z^{\scalebox{0.5}{$\cancel{\hookrightarrow}$}})} \notag\\
& >  -\log \left\{ {e^{ - \beta_{\hat{\lambda}} }} \right\} + \sum_{z^{\scalebox{0.5}{$\cancel{\hookrightarrow}$}} \in \Z^{\cancel{\hookrightarrow}}} P(g({Y^{e_{\hat{\lambda}}}})= z^{\scalebox{0.5}{$\cancel{\hookrightarrow}$}} ) \cdot \int-\log p^{*, \hat{\lambda}}  (Y^* | X^{*,\hat{\lambda}}, Y^* \in  g^{-1}(z^{\scalebox{0.5}{$\cancel{\hookrightarrow}$}})) dP_{(X^*,Y^*) | Y^* \in g^{-1}(z^{\scalebox{0.5}{$\cancel{\hookrightarrow}$}})} \notag \\
&  =  \beta_{\hat{\lambda}}+ \sum_{z^{\scalebox{0.5}{$\cancel{\hookrightarrow}$}} \in \Z^{\cancel{\hookrightarrow}}} P(g({Y^{e_{\hat{\lambda}}}})= z^{\scalebox{0.5}{$\cancel{\hookrightarrow}$}} ) \cdot \int-\log p^{*, \hat{\lambda}}  (Y^* | X^{*,\hat{\lambda}}, Y^* \in  g^{-1}(z^{\scalebox{0.5}{$\cancel{\hookrightarrow}$}})) dP_{(X^*,Y^*) | Y^* \in g^{-1}(z^{\scalebox{0.5}{$\cancel{\hookrightarrow}$}})} \notag\\
&  =  H(Y^* | X_1^*) - \sum_{z^{\scalebox{0.5}{$\cancel{\hookrightarrow}$}} \in \Z^{\cancel{\hookrightarrow}}}\left\{P(g(Y^*) = z^{\scalebox{0.5}{$\cancel{\hookrightarrow}$}}) \cdot \int \log p^{*, \hat{\lambda}}(Y^* |  X^{*,\hat{\lambda}}, Y^* \in g^{-1}(z^{\scalebox{0.5}{$\cancel{\hookrightarrow}$}} ) )dP_{(X^*,Y^*) | Y^* \in  g^{-1}(z^{\scalebox{0.5}{$\cancel{\hookrightarrow}$}})}\right\} \notag \\
&~~~~~~~~^+ \sum_{z^{\scalebox{0.5}{$\cancel{\hookrightarrow}$}} \in \Z^{\cancel{\hookrightarrow}}} P(g({Y^{e_{\hat{\lambda}}}})= z^{\scalebox{0.5}{$\cancel{\hookrightarrow}$}} ) \cdot \int-\log p^{*, \hat{\lambda}}  (Y^* | X^{*,\hat{\lambda}}, Y^* \in  g^{-1}(z^{\scalebox{0.5}{$\cancel{\hookrightarrow}$}})) dP_{(X^*,Y^*) | Y^* \in g^{-1}(z^{\scalebox{0.5}{$\cancel{\hookrightarrow}$}})}. \label{eq:Lem:min is no  II:2 }\\ \notag
  \end{align}
  Since $P_{X_1^{e_{\hat{\lambda}}}, g(Y^{e_{\hat{\lambda}}})} = P_{X_1^*, g(Y^*)} $($= P_{X_1^I, g(Y^I)} $) holds by the construction of $T_{all}$, we can see that $P(g(Y^{e_{\hat{\lambda}}})=  z) = P(g(Y^*)=z)$ for any $z \in \Z$. Therefore, we obtain 
  \begin{align*}
  &- \sum_{z^{\scalebox{0.5}{$\cancel{\hookrightarrow}$}} \in \Z^{\cancel{\hookrightarrow}}}\left\{P(g(Y^*) = z^{\scalebox{0.5}{$\cancel{\hookrightarrow}$}}) \cdot \int \log p^{*, \hat{\lambda}}(Y^* |  X^{*,\hat{\lambda}}, Y^* \in g^{-1}(z^{\scalebox{0.5}{$\cancel{\hookrightarrow}$}} ) )dP_{(X^*,Y^*) | Y^*  \in g^{-1}(z^{\scalebox{0.5}{$\cancel{\hookrightarrow}$}})}\right\} \notag \\
&~~~~~~~~^+ \sum_{z^{\scalebox{0.5}{$\cancel{\hookrightarrow}$}} \in \Z^{\cancel{\hookrightarrow}}} P(g({Y^{e_{\hat{\lambda}}}})= z^{\scalebox{0.5}{$\cancel{\hookrightarrow}$}} ) \cdot \int-\log p^{*, \hat{\lambda}}  (Y^* | X^{*,\hat{\lambda}}, Y^* \in  g^{-1}(z^{\scalebox{0.5}{$\cancel{\hookrightarrow}$}})) dP_{(X^*,Y^*) | Y^* \in g^{-1}(z^{\scalebox{0.5}{$\cancel{\hookrightarrow}$}})}  =0. \\
  \end{align*}
 Hence, we can conclude (\ref{eq:Lem:min is no  II:2 }) $= H(Y^* | X_1^*)$.  Moreover, $H(Y^* | X_1^*) = \R^{o.o.d.}(\lambda^I)$; indeed, 
 \begin{align}
  \R^{o.o.d.} (\lambda^I) &= \max_{(X, Y) \in _{all}} \R^{(X, Y) } (p^{*, \lambda^I}\circ \Phi^{\lambda^I}) = \max_{(X, Y) \in T_{all}} \int - \log p^{*, \lambda^I} (Y | \Phi^{\lambda^I}(X)) dP_{X, Y} \notag \\
  &=  \max_{(X, Y) \in T_{all}} \int - \log p^{*, \lambda^I} (Y | \Phi^{\lambda^I}(X)) dP_{\Phi^{\lambda^I}(X), Y} \notag\\
  &=  \max_{(X, Y) \in T_{all}} \int - \log p^{*, \lambda^I} (Y |X_1) dP_{X_1, Y} \label{eq:Lem:min is no  II:3_2 }
 \end{align}
 Noting that  $P_{X_1, Y}  =P_{X_1^*, Y^*} (=  P_{X_1^I, Y^I})$ for any $(X, Y) \in T_{all}$ and   $p^{*, \lambda^I}$ coincides with the conditional probability density function of $P_{Y^* | X_1^*}$, we can see that
 $$ \mathrm{(\ref{eq:Lem:min is no  II:3_2 })}=   \int - \log p^{*, \lambda^I} (Y^* |X_1^*) dP_{X_1^*, Y^*}  = H (Y^*| X_1^*).$$
 
 Hence, we can derive (\ref{eq:Lem:min is no  II:2 }) $= H (Y^*| X_1^*) = \R^{o.o.d.}(\lambda^I)$, which concludes the proof. $\Box$\\
   {\bf{proof of Lemma \ref{Lem:no  II is H=0_2}}}.\\
   Take $\hat{\lambda} \in \Lambda$ which satisfies $\mathrm{Im} \Phi_2^{\hat{\lambda}} =  \emptyset$.
   Since $\mathrm{Im} \Phi_2^{\hat{\lambda}} =  \emptyset$, $P_{\Phi(X^{\hat{\lambda}}),Y} = P_{\Phi(X^{ *,\hat{\lambda}}), Y^*}$ holds for $\forall (X, Y) \in T_{all}$ because of the construction of $T_{all}$.
   Therefore, 
   
   \begin{align*}
   \R^{o.o.d.}(\hat{\lambda})   &:= \max_{(X,Y) \in T_{all}} \R^{(X,Y)}(p^{*, \hat{\lambda}} \circ \Phi^{\hat{\lambda}})= \max_{(X,Y) \in T_{ad} \cup \{ (X^*, Y^*)\}} \R^{(X,Y)}(p^{*, \hat{\lambda}} \circ \Phi^{\hat{\lambda}})\\
  & \hspace{-27mm}=  \max_{(X,Y) \in T_{ad} \cup \{ (X^*, Y^*)\}} \Biggl\{   \R^{(X,g(Y))}(p^{*, \hat{\lambda}} \circ \Phi^{\hat{\lambda}})+  \sum_{z^{\scalebox{0.5}{$\cancel{\hookrightarrow}$}} \in \Z^{\cancel{\hookrightarrow}}}  \left\{P(g(Y)= z^{\scalebox{0.5}{$\cancel{\hookrightarrow}$}} )  \int-\log p^{*, \hat{\lambda}} (Y | X^{{\hat{\lambda}}}, Y \in g^{-1}(z^{\scalebox{0.5}{$\cancel{\hookrightarrow}$}})) dP_{(X,Y) | Y\in g^{-1}(z^{\scalebox{0.5}{$\cancel{\hookrightarrow}$}})} \right\}\Biggl\} \\
  & \hspace{-28mm}=  \max_{(X,Y) \in T_{ad} \cup \{ (X^*, Y^*)\}}\Biggl\{   \R^{(X,g(Y))}(p^{*, \hat{\lambda}} \circ \Phi^{\hat{\lambda}})+ \sum_{z^{\scalebox{0.5}{$\cancel{\hookrightarrow}$}} \in \Z^{\cancel{\hookrightarrow}}}  \left\{P(g(Y)= z^{\scalebox{0.5}{$\cancel{\hookrightarrow}$}} )  \int-\log p^{*, \hat{\lambda}} (Y^* | X^{*,{\hat{\lambda}}}, Y^* \in g^{-1}(z^{\scalebox{0.5}{$\cancel{\hookrightarrow}$}})) dP_{(X^*,Y^*) | Y^*\in g^{-1}(z^{\scalebox{0.5}{$\cancel{\hookrightarrow}$}})} \right\}\Biggl\} \\
   &~~~~~~~~~~~~~~~~~~~~~~~~~~~~~~~~~~~~~~~~~~~~~~~~~~~~~~~~~~~~~~~~~~~~~~~~~~~~~~~~~~~~~~~~~~~~~~~~~~~~~~~~~~~~~~~ = \R^{II} (\hat{\lambda}) 
\end{align*}
and therefore, it concludes the proof. Here, the third equality holds by Theorem \ref{thm:HIT-CV}; the first and forth equalities hold because of the fact that $P_{ \Phi(X^{\hat{\lambda}}), Y} = P_{\Phi(X^{*, \hat{\lambda}}), Y^*}$ holds for $\forall (X, Y) \in T_{all}$.

\section{Sufficient Conditions of $e^*$ for there to exist $(X^{e_{\lambda} }, Y^{e_{\lambda} })$ which satisfies  (d) and (d)'}\label{append:condi_(c)(c)'}
\begin{theo} \label{theo:exits_env1}
Notations are same as in Theorem \ref{Thm:justfyHITCV1:re} and \ref{Thm:justfyHITCV2:re}.
$(X^*, Y^*)$ satisfies the following condition:

\begin{itemize}
\item[(A)] For a sufficiently small $\e \ll 1$, the following statement holds:\\
$\forall \lambda$ with $\mathrm{Im} \Phi^{\lambda}_2 \neq \emptyset$, $\forall \a \in \mathrm{Im} \Phi^{\lambda}_1$, $\forall b \in \Y$, $\exists c(\lambda, a,b)$\footnote{$c(\lambda, a,b)$ means $c\in \X_2$ is determined by given $\lambda \in \Lambda$, $a\in \X_1$, $b \in \Y$.} s.t. $P(Y^* = b  | X_1^{* , \lambda} = a, X_2^{*, \lambda} =c  ) \geq (1-e^{-\beta}) + \e$. \\
\end{itemize}
Then, $\forall \lambda$ with $\mathrm{Im} \Phi^{\lambda}_2 \neq \emptyset$,  there exists $(X^{e_\lambda}, Y^{e_\lambda}) \in T_{all}$ such that the inequality in (d) holds.
\end{theo}

\hspace{-5mm}{\bf{Remark}}. The condition (A) means that, in the environment $e=e^*$, the affection of environmental factors ($= \X_2$) to the response variable $Y^{e^*}$ is large; indeed, the inequality in (A) means that, if $\lambda$ fails to remove environment factors ($i.e.$, $\mathrm{Im} \Phi^{\lambda}_2 \neq \emptyset$), we can control the probability of $Y^{e^*}=b$ by the selection $c$ for any $b \in \Y$.

\begin{proof}
Fix $\forall \lambda$ with $\mathrm{Im} \Phi^{\lambda}_2 \neq \emptyset$.
Take $(\bar{X},\bar{Y}) \in T_{all}$ such that its probability measure corresponds to  $\bar{P}_{X_2 | Y,X_1} \times P_{Y^I,X_1^I},$
 where $\bar{P}_{X_2 | Y,X_1}$ is defined by, setting $\hat{c}(\hat{\lambda}, a, b)$ by 
 $$\hat{c}(\hat{\lambda}, a, b) \in  \argmin_{ c \in \X_2 } P (g(Y^{*}) = g(b) | X^{*,\hat{\lambda}}_1 = \Phi_1^{\hat{\lambda}} (a), X^{*,\hat{\lambda}}_2= \Phi_2^{\hat{\lambda}}(c)),$$
 $\bar{P}_{X_2 | Y=b,X_1=a}:= \delta_{X_2 = \hat{c}(\hat{\lambda}, a,b) }.$ Here, for $c \in \X_2$, the probability measure $\delta_{X_2= c}$ on $\X_2$ denotes a Dirac measure at $c \in \X_2$.
 
 Before proving Theorem \ref{theo:exits_env1}, we prepare the following inequalities:\\
  {\bf{Supplementary Inequality 1}}.\\
 $$ \forall a \in \X_1, \forall b \in \Y,    P \left( g(Y^{*}) = g(b) \left| X^{*,\hat{\lambda}}_1 = \Phi_1^{\hat{\lambda}} (a), X^{*,\hat{\lambda}}_2= \Phi_2^{\hat{\lambda}}  \bigl( \hat{c} ( \hat{\lambda}, a, b )  \bigr)  \right. \right) \leq {e^{ - \beta }- \epsilon } .$$
 To see the fact, take $b^* \in \Y$ such that $g(b^*) \neq g(b)$\footnote{Such $b^*$ always exists by the following reason. Since $| \Z | \geq 2$ (which is imposed on in Chapter \ref{sec.setting} ), we can take $\Z \ni z^* \neq g(b)$. By the surjectivity of $g$, $g^{-1}(z^*) \neq \emptyset$. Taking $b^* \in g^{-1}(z^*)$, $g(b^*) = z^* \neq g(b)$.}. Then, by the condition (ii) of Theorem \ref{Thm:justfyHITCV2} and $\mathrm{Im} \Phi_2^{\hat{\lambda}} \neq \emptyset $, there exists $c(\hat{\lambda}, a, b ) \in \X_2$ such that 
$$ P  \left( Y^{*} = b^* \left| X^{*,\hat{\lambda}}_1 = \Phi_1^{\hat{\lambda}} (a), X^{*,\hat{\lambda}}_2= \Phi_2^{\hat{\lambda}}  \bigl( c(\hat{\lambda}, a, b ) \bigr)  \right. \right)  \geq {1-e^{ - \beta }+ \epsilon } .$$
 Therefore, 
  \begin{align*}
 &P \left( g(Y^{*}) = g(b) \left| X^{*,\hat{\lambda}}_1 = \Phi_1^{\hat{\lambda}} (a), X^{*,\hat{\lambda}}_2= \Phi_2^{\hat{\lambda}}  \bigl( \hat{c}(\hat{\lambda}, a, b) \bigr)  \right. \right)\\
 & = \min_{c \in \X_2}P ( g(Y^{*}) =g( b)  | X^{*,\hat{\lambda}}_1 = \Phi_1^{\hat{\lambda}} (a), X^{*,\hat{\lambda}}_2= \Phi_2^{\hat{\lambda}}(c))\\
  & \leq P  \left( g(Y^{*}) = g(b) \left| X^{*,\hat{\lambda}}_1 = \Phi_1^{\hat{\lambda}} (a), X^{*,\hat{\lambda}}_2= \Phi_2^{\hat{\lambda}}  \bigl( c(\hat{\lambda}, a, b ) \bigr)  \right. \right) \\
  & = 1 - \sum_{\bar{z} \neq g(b) }P  \left(g(Y^{*}) = \bar{z} \left| X^{*,\hat{\lambda}}_1 = \Phi_1^{\hat{\lambda}} (a), X^{*,\hat{\lambda}}_2= \Phi_2^{\hat{\lambda}} \bigl( c(\hat{\lambda}, a, b )\bigr)  \right. \right) \\
  & \leq 1 - P  \left(g(Y^{*}) = g(b^*) \left| X^{*,\hat{\lambda}}_1 = \Phi_1^{\hat{\lambda}} (a), X^{*,\hat{\lambda}}_2= \Phi_2^{\hat{\lambda}} \bigl( c(\hat{\lambda}, a, b ) \bigr)  \right.  \right) \\
        & \leq 1 -  P  \left(Y^{*} = b^* \left| X^{*,\hat{\lambda}}_1 = \Phi_1^{\hat{\lambda}} (a), X^{*,\hat{\lambda}}_2= \Phi_2^{\hat{\lambda}} \bigl( c(\hat{\lambda}, a, b ) \bigr) \right. \right). \\
   & \leq 1- (1-e^{ - \beta }+ \epsilon) \\
   & \leq {e^{ - \beta }- \epsilon }. \\
    \end{align*}
     {\bf{Proof of Theorem \ref{theo:exits_env1}}} \\
     We may prove that $P_{\bar{X}, \bar{Y}} (A) = 1$ where  
     $$ \Biggl\{ (x, y) \in \X \times \Y  \Biggr| P \left( g(Y^{*}) = g(b) \left| X^{*,\hat{\lambda}} = \Phi^{\hat{\lambda}} (x)   \right) \leq {e^{ - \beta }- \epsilon } \right.  \Biggr\}.$$
     Then, 
     \begin{align*}
          P_{\bar{X}, \bar{Y}} (A) &=  \int 1_{A} d P_{\bar{X}, \bar{Y}}  =  \int 1_{A} d (\bar{P}_{X_2 | Y,X_1} \times P_{Y^I,X_1^I}) \\
          &= \int d P_{Y^I,X_1^I} \int 1_{A} d \bar{P}_{X_2 | Y,X_1} =  \int d P_{Y^I,X_1^I} (x_1, y)  \delta_{X_2 = \hat{c}(\hat{\lambda}, x_1,y) }(A_{(x_1, y)})
     \end{align*}
     holds where $A_{(x_1, y)} := \left\{ x_2 \in \X_2 | ((x_1, x_2), y) \in \X \times \Y \right\}$. By the Supplementary Inequality 1, $\hat{c}(\hat{\lambda}, x_1,y)  \in A_{(x_1, y)} $ holds and therefore, 
     $\delta_{X_2 = \hat{c}(\hat{\lambda}, x_1,y)} (A_{(x_1, y)}) =1$, which leads us to the equation  $$\int d P_{Y^I,X_1^I} (x_1, y)  \delta_{X_2 = \hat{c}(\hat{\lambda}, x_1,y) } (A_{(x_1, y)}) =1$$.
    
       \end{proof}

\begin{theo} \label{theo:exits_env2}
Notations are same as in Theorem \ref{Thm:justfyHITCV1:re} and \ref{Thm:justfyHITCV2:re}.
$(X^*, Y^*)$ satisfies the following condition:

\begin{itemize}
\item[(A)'] For a sufficiently small $\e \ll 1$, the following statement holds:\\
$\forall \lambda$ with $\mathrm{Im} \Phi^{\lambda}_2 \neq \emptyset$, $\forall \a \in \mathrm{Im} \Phi^{\lambda}_1$, $\forall b \in \Y$, $\exists c(\lambda, a,b)$ s.t. $P(Y^* = b  | X_1^{* , \lambda} = a, X_2^{*, \lambda} =c  ) \geq (1-e^{-\beta_{\lambda}}) + \e$. \\
\end{itemize}
Then, $\forall \lambda$ with $\mathrm{Im} \Phi^{\lambda}_2 \neq \emptyset$, there exists $(X^{e_\lambda}, Y^{e_\lambda}) \in T_{all}$ such that the inequality in (d)' holds.
\end{theo}
The proof of Theorem \ref{theo:exits_env2} is essentially same as the one of Theorem \ref{theo:exits_env1} and therefore, we omit.

\newpage
\section{Additional Experiment explanations}\label{app:add_exp_detal}

\subsection{Visualizations of Experiment  Results}\label{subsc:exp.rep.of.Syn}
Synthesized data of first experiment is visualized as in Figure \ref{fig:VisSyn1}. Synthesized data of second experiment is visualized as in Figure \ref{fig:VisSyn2}.
\subsection{Explicit representation of Second Synthetic data}
\begin{align*}
 N_1 & = \N (-180, 20^2)  \times \N (-5e, 30^2),  \\
 N_2 & = \N (-100, 20^2)\times \N (-3e, 30^2) , \\
 N_3 & = \N (-20, 20^2)\times \N (-1e, 30^2) ,\\
 N_4 & = \N (60, 20^2) \times \N (-2e, 30^2), \\
 N_5 & = \N (140, 20^2)\times \N (-4e, 30^2) ,  \\
 N_6 & = \N (-140, 20^2) \times \N (4e, 30^2), \\
 N_7 & = \N (-60, 20^2) \times \N (2e, 30^2),  \\
 N_8 & = \N (20, 20^2)\times \N (1e, 30^2) , \\
 N_9 & = \N (100, 20^2) \times \N (3e, 30^2),\\
 N_{10} & = \N (180, 20^2) \times \N (5e, 30^2).
\end{align*}

  \begin{figure}
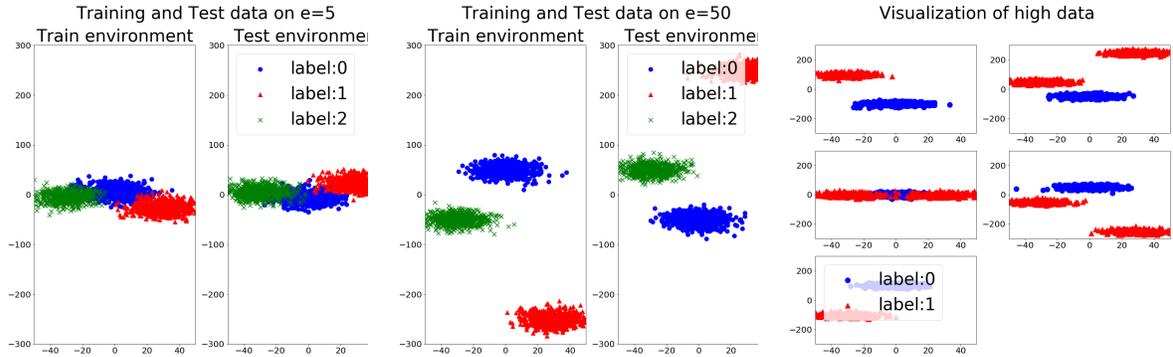

\centering
\caption{Visualization of Synthesized Data 1. Left and middle figures illustrate training and test data on $e^*=5$ and $50$, respectively. As $e^*$ increases, the test data and train data are more different, and therefore ERM yields lower performance. Right figure illustrates $\D_{ad}^e$.  }\label{fig:VisSyn1}
\vspace{3mm}
\resizebox{\textwidth}{!}{
    \begin{tabular}{l}
   \begin{minipage}[b]{0.4\linewidth}
    \centering
    \includegraphics[keepaspectratio, scale=0.2]{Syne=5.pdf}
  \end{minipage}
  \begin{minipage}[b]{0.4\linewidth}
    \centering
    \includegraphics[keepaspectratio, scale=0.2]{Syne=50.pdf}
        \end{minipage}
      \begin{minipage}[b]{0.4\linewidth}
    \centering
    \includegraphics[keepaspectratio, scale=0.2]{high_data.pdf}
     \end{minipage}
   \end{tabular}  }
  \end{figure}
  
    \begin{figure}
\centering
\caption{Visualization of Synthesized Data 2. Left figure illustrates the training and test data of second experiment.  Right figure illustrates $\D_{ad}^{40}$ and $\D_{ad}^{e_{ad}}$ with $e_{ad} = -9$.  }\label{fig:VisSyn2}
\vspace{3mm}
\resizebox{\textwidth}{!}{
    \begin{tabular}{l}
   \begin{minipage}[b]{0.5\linewidth}
    \centering
    \includegraphics[keepaspectratio, scale=0.2]{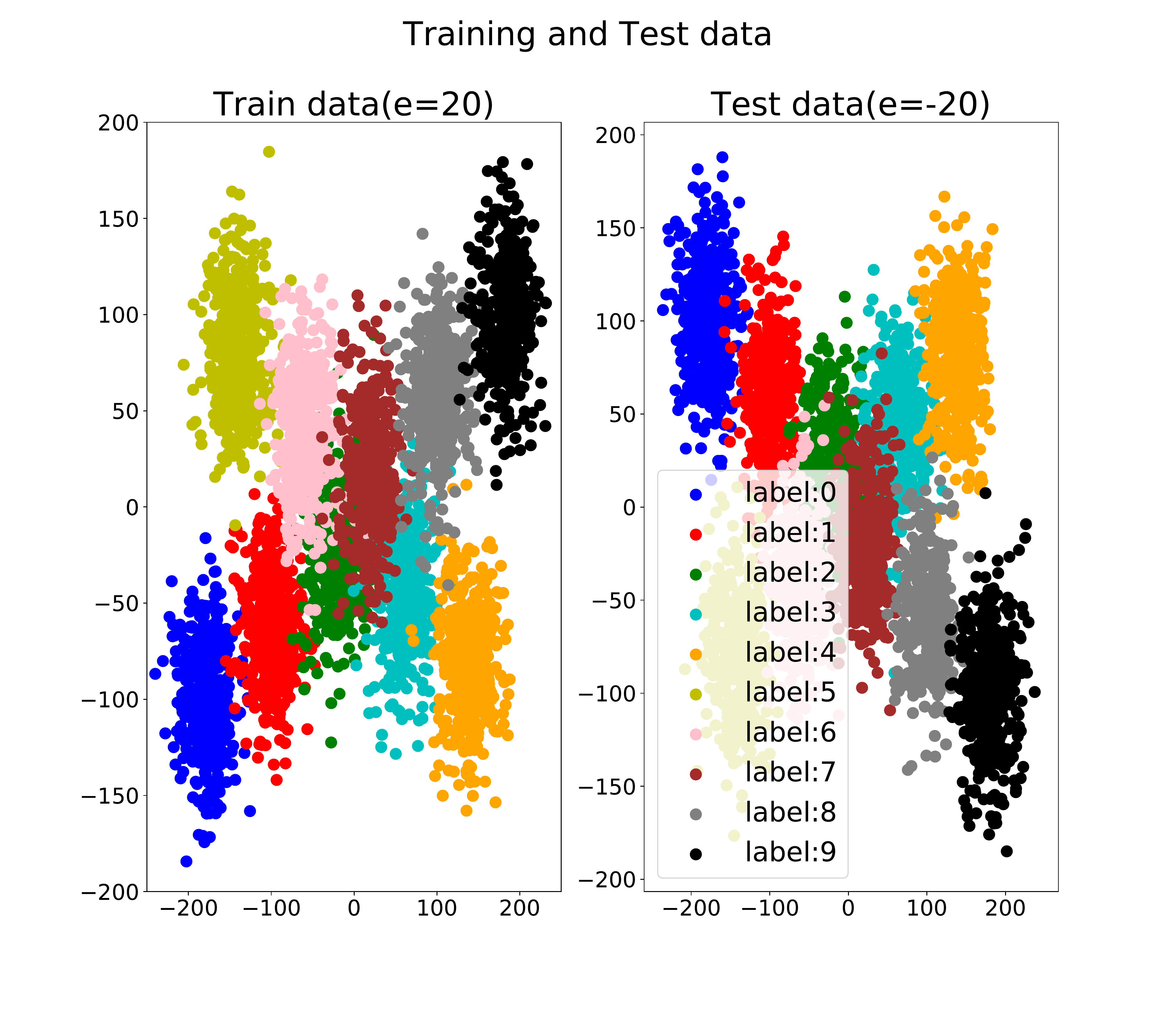}
  \end{minipage}
      \begin{minipage}[b]{0.5\linewidth}
    \centering
    \includegraphics[keepaspectratio, scale=0.2]{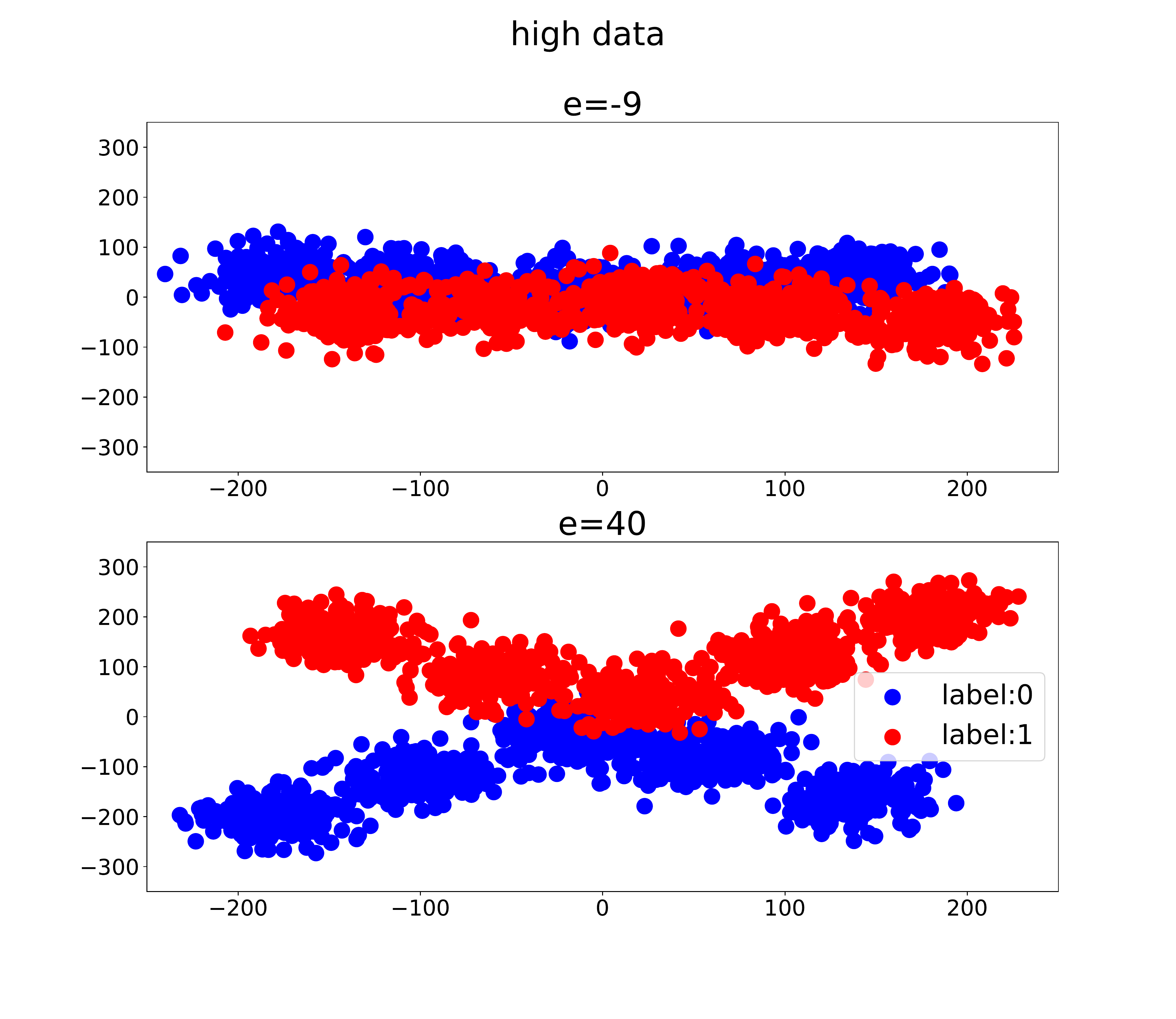}
     \end{minipage}
   \end{tabular}  }
  \end{figure}

\begin{table*}[t]
\centering
\tiny
 \caption{Test Acc. of  Hierarchical Colored MNIST (5runs) }\label{Fig:acc Colored MNIST}
  \vspace{3mm}
\centering
\resizebox{\textwidth}{!}{
 \scalebox{0.2}{ \begin{tabular}{|c|c|c|c|c|c|c|c||c|c||c|} \hline
   flip rate& Test Acc. on&  Best passible  &  Oracle &  ERM &  FT &  FE &  DSAN & Ours +  CVI & Ours +  CVII & Ours+ TDV \\ \hline 
     \multirow{2.5}{*}{0.25}&$e=0.1$&\multirow{2.5}{*}{.750}&\multirow{2.5}{*}{.729 (.004)} &.\textbf{.771 (.001)}  & .\textbf{771 (.001)} &  .\textbf{.771 (.001)}& .767 (.004)&  .727 (.004)&  .714 (.013)&  .673 (.006) \\ 
      \rule[-4pt]{0pt}{10pt}
     & $e=0.9$ &&  &.125 (.003)   & .128 (.002)&  $.131 (.002)$ & .085 (.003)& .622 (.015)    &  \textbf{ .644 (.019)}     &  .690 (.009)\\  \hline 
      \rule[-4pt]{0pt}{10pt}
     \multirow{2.5}{*}{0.20}&$e=0.1$ &\multirow{2.5}{*}{.800}&\multirow{2.5}{*}{.780 (.002)}&.796 (.000)  & .\textbf{800 (.001)} &  .796 (.001)& .789 (.004)&  .773 (.003)& .745 (.008)& .738	(.018) \\
      \rule[-4pt]{0pt}{10pt}
     & $e=0.9$ &&  &.177 (.006)   & .201 (.004)&  $.200 (.007)$ &.091 (.005)&  .644 (.011)& \textbf{ .707 (.012)} &  .732 (.008)\\ \hline
      \rule[-4pt]{0pt}{10pt}
     \multirow{2.5}{*}{0.15}&$e=0.1$&\multirow{2.5}{*}{.850}&\multirow{2.5}{*}{.828 (.004)} &.822 (.000)  & .823 (.001) &  \textbf{.824 (.002)}&.815 (.002) &  .814	(.007) &.797 (.011) &  .822	(.001)\\
      \rule[-4pt]{0pt}{10pt}
     &$e=0.9$&&  &.277 (.007)   & .323 (.006)&  $.312 (.012)$ &.091 (.002)& .724 (.037) &\textbf{ .743 (.020)} & .782	(.012)\\ \hline
      \rule[-4pt]{0pt}{10pt}
     \multirow{2.5}{*}{0.10}& $e=0.1$ &\multirow{2.5}{*}{.900}&\multirow{2.5}{*}{.880 (.004)}&.852 (.002)  & .855(.001) & .\textbf{ .856 (.001)}&.833(.003) &  .848 (.005) & .848 (.005) &   .857 (.005)  \\
      \rule[-4pt]{0pt}{10pt}
     &$e=0.9$&&  &.468 (.002)  & .497 (.005)&  $.500 (.007)$ &.106 (.010)& \textbf{ .792 (.005)}& \textbf{ .792	(.005)}&  .829 (.005)\\ \hline
  \end{tabular}}}
\end{table*}

\subsection{Additional Experiment of Hierarchical Colored MNIST}\label{append:additionalMNIST}
Although Hierarchical Colored MNIST in Chapter 6 fix its flip rate $20 \%$, we additionally demonstrate by changing its flip rate among $\{0.1, 0.15, 0.2, 0.25 \}$. Table \ref{Fig:acc Colored MNIST} shows that our methods outperform other methods. Table \ref{Fig:acc Colored MNIST} and \ref{tb:wraptable} show that, among several CV methods, our method II outperform other methods. 
Table \ref{Fig:acc Colored MNIST CVdiff} the difference between accuracies by TDV and each CV for the same data set with $e=0.9$.  The result verify that CVII selects preferable hyperparameters with smaller errors.

 \begin{table}
  \centering
  \caption{Baselines of CV methods}
  \vspace{5mm}
 \begin{tabular}{|c|c|c|} \hline
     &  \footnotesize Tr-CV &\footnotesize  LOD-CV\\ \hline \hline
    \multirow{2.}{*}{ \footnotesize0.25}& \footnotesize.759 (.008)& \footnotesize .362	(.059)\\
    &.459 (.012) &\footnotesize  .372 (.037)\\ \hline
       \multirow{2.}{*}{ \small0.20}&.794 (.004) &  \footnotesize.338 (.048) \\
     &.541 (.007) &\footnotesize  .334 (.029) \\ \hline
   \multirow{2.}{*}{ \small0.15}& .834 (.002) & .348 (.031) \\
   &.634 (.008) &.358 (.024)\\ \hline
    \multirow{2.}{*}{ \small0.10} & .876 (.003)&.502 (.196)\\
    &.708 (.006) &.497 (.194)\\ \hline
  \end{tabular}
  \label{tb:wraptable}
\end{table}

   \begin{table}[t]
\centering
\tiny
 \caption{Means and SEs of $\{$({\text Accuracy of TDV on $e=0.9$}) -({\text  Accuracy of Each CV on $e=0.9$}) $\}$ (5runs). }\label{Fig:acc Colored MNIST CVdiff}
 \vspace{3mm}
\resizebox{0.6\textwidth}{!}{
 \scalebox{0.1}{ \begin{tabular}{|c|c|c|c|c|} \hline
     &   CV I & CV II &   Tr-CV &  LOD-CV\\ \hline \hline
    0.25& .068 (.007)&  .046	(.023)& .231 (.013)&  .319	(.033)\\ \hline
       0.20&.088 (.004) &  .025 (.006)&.191 (.014) &  .398 (.025)\\ \hline
  0.15& .059 (.038) & .039 (.022)&  .148 (.019) & .430 (.028)  \\ \hline
    0.10 & .037 (.010)&.037 (.010)  & .121 (.008)&.332 (.196) \\ \hline
  \end{tabular}}}\
 \end{table}

\subsection{Experiment Details}
Through the experiment in the present paper, all models of competitors are composed of neural networks where its loss function, activation function, and optimizer are cross entropy, Relu Networks and Adam \cite{Diederik}. In the following explanation, NN with its model architecture $a \rightarrow h_1 \rightarrow \cdots h_k \rightarrow h_n \rightarrow \P_{[m]}$ means that its input and hidden dimensions are $a$ and $(h_1,...,h_)$ respectively, and its output is probability density functions on $[m]$. NN with its model architecture $a \rightarrow h_1 \rightarrow \cdots h_k \rightarrow h_n \rightarrow b$ means that its input, hidden and output dimensions are $a$, $(h_1,...,h_n)$ and $b$ respectively.  All the experiment, we add $L^2$-reguralized term to our objective function.


We add explanations of previous CV methods. Tr-CV implements cross-validation with using only $\D^*$. In LOD-CV, a model is learnt with  excluding one of the $\D^e \in \D^{ad}$ from $\D^{ad}$, and evaluate its performance by $\D^e$. Changing the role of $e \in \E^{ad}$, and taking their mean, we evaluate final CV-value.

\subsubsection{Synthesized Data I}
We set model architecture of $\Phi$ used in our method $2 \rightarrow 20 \rightarrow 20 \rightarrow 1$. We set model architecture of CB-ERM and ERM  $2 \rightarrow 20 \rightarrow 20 \rightarrow \P_{[3]}$. When we use FT and FE, its model architecture on pre-train phase and retraining phase are $2 \rightarrow 20 \rightarrow 20 \rightarrow \P_{[2]}$ and $2 \rightarrow 20 \rightarrow 20 \rightarrow \P_{[3]}$ respectively. We set running rate and hyperparameters of $L^2$-regularized term $0.0115$ and $0.01$ respectively. When we use $DSAN$ \cite{Stojanov}, we inherit learning condition in the {\em Amazon Review dataset} experiment. When training,  we use batch learning. We set $K=10$ of each CV method.

\subsubsection{Synthesized Data II}
We set model architecture of $\Phi$ used in our method $2 \rightarrow 8 \rightarrow 8 \rightarrow 1$. We set running rate and hyperparameters of $L^2$-regularized term $0.05$ and $0.001$ respectively. When training, we use minibatch learning with dividing  $\D^*$, $\D^{e^{ad}}$ and $\D^{40}$ into 50 equal parts respectively. We set $K=10$ of each CV method.

\subsubsection{Hierarchical Colored MNIST}
We set model architecture of $\Phi$ used in our method $2 \rightarrow 440 \rightarrow 440 \rightarrow 440$. We set model architecture of CB-ERM and ERM  $2 \rightarrow 440 \rightarrow 440 \rightarrow \P_{[3]}$. When we use FT and FE, its model architecture on pre-train phase and retraining phase are $2 \rightarrow 440 \rightarrow 440 \rightarrow \P_{[2]}$ and $2 \rightarrow 440 \rightarrow 440 \rightarrow \P_{[3]}$ respectively. We set running rate and hyperparameter of $L^2$-regularized term $0.0004$ and $0.002$ respectively. When we use $DSAN$, we inherit learning condition in the {\em Amazon Review dataset} experiment. When training,  we use batch learning. We set $K=10$ of our CV method.

\subsubsection{Birds recognition}
We set model architecture of $\Phi$ used in our method ResNet50 \cite{K. He} with changing its output dimension $256$. We set model architecture of CB-ERM and ERM  ResNet50 \cite{K. He} with changing its output  $\P_{[3]}$. When we use FT and FE, its model architecture on pre-train phase and retraining phase are ResNet50 \cite{K. He} with changing its output dimension $2$ and $3$ respectively. We set running rate and hyperparameter of $L^2$-regularized term $0.00004$ and $0.001$ respectively. When training,  we use minibatch learning with its minibatch size $56$. We set $K=5$ of each CV method.






\begin{thebibliography}{99}
    
    \bibitem{Vladimir}
V. Vapnik.
\newblock Principles of risk minimization for learning theory.
\newblock \emph{In NIPS}, 1992.

\bibitem{adolphs2019ellipsoidal}
L. Adolphs, J. Kohler, and A. Lucchi.
\newblock Ellipsoidal trust region methods and the marginal value of hessian
  information for neural network training.
\newblock \emph{arXiv:1905.09201 (version 1)}, 2019.


\bibitem{S. Beery}
S. Beery, G. V. Horn, P. Perona.
\newblock Recognition in terra incognita.
\newblock \emph{In ECCV}, 2019.

\bibitem{Zech}
 J. R. Zech, M.A. Badgeley, M. Liu, A. B. Costa, J. J. Titano, E. K. Oermann.
\newblock Variable generalization performance of a deep learning model to detect pneumonia in chest radiographs: A cross-sectional study. 
\newblock \emph{$PLoS$ Medicine 15, e1002683 }, 2018.


 \bibitem{Niven} 
 T. Niven, H.-Y Kao.
\newblock Probing neural network comprehension of natural language arguments.
\newblock \emph{Proceedings of the 57th Annual Meeting of the Association for Computational}, 2019. 


\bibitem{Gururangan} 
S. Gururangan, S. Swayamdipta, O. Levy, R. Schwartz, S. R. Bowman, N. A. Smith.
\newblock Annotation artifacts in Natural Language Inference data.
\newblock \emph{Proceedings of the Conference of the North American Chapter of the Association for Computational Linguistics} 2018.

\bibitem{Dastin} 
J. Dastin.
\newblock Amazon scraps secret AI recruiting tool that showed bias against women.
\newblock \emph{https://reut.rs/2Od9fPr.} 2018.

\bibitem{Ilyas}
 A. Ilyas, S. Santurkar, D. Tsipras, L. Engstrom, B. Tran, A. Madry.
 \newblock Adversarial examples are not bugs, they are features,.
\newblock \emph{In NeurIPS}, 2019.


\bibitem{Shane} 
J. Shane.
\newblock Do neural nets dream of electric sheep? 
\newblock \emph{https://aiweirdness.com/post/171451900302/ },  2018.

\bibitem{Martin} 
M. Arjovsky, L. Bottou, I. Gulrajani, and D. Lopez-Paz.
 \newblock  Invariant Risk Minimization.
\newblock \emph{ arXiv:1907.02893}, 2019.
    
\bibitem{Jonas} 
J. Peters, P. B\"{u}hlmann, and N. Meinshausen. 
 \newblock Causal inference using invariant prediction: identification and confidence intervals. 
\newblock \emph{JRSS~B,} 2016.

\bibitem{Christina} 
C. Heinze-Deml, J. Peters, and N. Meinshausen.
\newblock   Invariant causal prediction for nonlinear models.
\newblock \emph{Journal ~of ~Causal ~Inference}, 2018.

\bibitem{Rothenhausler} 
D. Rothenhausler, P. Buhlmann, N. Meinshausen, and J. Peters. 
\newblock  Anchor regression: heterogeneous data meets causality.
\newblock \emph{JRSS B}, 2018.

\bibitem{Kartik} 
K. Ahuja, K. Shanmugam, K. Varshney, and A. Dhurandhar. 
\newblock  Invariant risk minimization games.
\newblock \emph{In ICML}, 2020.

\bibitem{koyama} 
M. Koyama, S. Yamaguchi.
\newblock  When is invariance useful in an Out-of-Distribution Generalization problem?
\newblock \emph{arXiv}, 2008.01883, 2020



\bibitem{Ehab} 
E. A. AlBadawy, A. Saha, and M. A. Mazurowski. 
\newblock  Deep learning for segmentation of brain tumors: Impact of cross-institutional training and testing. 
\newblock \emph{Medical physics,} 2018.

\bibitem{Christian} 
C. S. Perone, P. Ballester, R. C. Barros, and J. Cohen-Adad. 
\newblock Unsupervised domain adaptation for medical imaging segmentation with self-ensembling. 
\newblock \emph{NeuroImage}, 2019.

\bibitem{Will} 
W. D. Heaven. 
\newblock Google's medical AI was super accurate in a lab. real life was a different story. 
\newblock \emph{MIT Technology Review,} 2020.

\bibitem{Ishaan} 
I. Gulrajani, D. Lopez-Paz.
\newblock In Search of Lost Domain Generalization.
\newblock \emph{In ICLR,} 2021.


\bibitem{Pritish} 
P. Kamath.
\newblock Does Invariant Risk Minimization Capture Invariance?
\newblock \emph{In AISTATS,} 2021.

\bibitem{David} 
D. Krueger, E. Caballero, J. Jacobsen, A. Zhang, J. Binas, D. Zhang, R. L. Priol, A. Courville.
\newblock Out-of-Distribution Generalization via Risk Extrapolation.
\newblock \emph{In ICML,} 2021.

\bibitem{Pan} 
S. J. Pan, and  Q. Yang.
\newblock A survey on transfer learning.
\newblock \emph{IEEE Transactions on Knowledge and Data Engineering
}, 22(10):1345-1359,  2009. 

\bibitem{Yang} 
Q. Yang, Y. Zhang, W. Dai, and S. J. Pan. 
\newblock Transfer Learning.
\newblock \emph{ Cambridge University Press,} 2020. 

\bibitem{Yosinski}
J. Yosinski, J. Clune, Y. Bengio,  and H. Lipson.
\newblock How transferable are features in deep neural networks?
\newblock \emph{In NIPS, } 2014.   

\bibitem{Yaroslav}
Y. Ganin, E. Ustinova, H. Ajakan, P. Germain, H. Larochelle, F. Laviolette, M. March, and V. Lempitsky. 
\newblock Domain-adversarial training of neural networks. 
\newblock \emph{Journal of Machine Learning Research}, 2016.


\bibitem{Shai}
S. Ben-David, J. Blitzer, K. Crammer, and F. Pereira. 
\newblock Analysis of representations for domain adaptation. 
\newblock \emph{In NIPS}, 2007.

\bibitem{Gilles}
G. Louppe, M. Kagan, and K. Cranmer. 
\newblock Learning to pivot with adversarial networks. 
\newblock \emph{In NIPS}, 2017.



\bibitem{S.Sagawa}
S. Sagawa, P. W. Koh, T. B. Hashimoto, P. Liang.
\newblock Distributionally Robust Neural Networks for Group Shifts: On the Importance of Regularization for worst-case generalization. 
\newblock \emph{In ECCV}, 2019.



\bibitem{K. He}
K. He, X. Zhang, S. Ren, and J. Sun.
\newblock Deep residual learning for image recognition, 
\newblock \emph{In CVPR}, 2016.

\bibitem{Fin}
C. Finn, P. Abbeel, and S. Levine.
\newblock Model-agnostic meta- learning for fast adaptation of deep networks.
\newblock \emph{In ICML}, 2017.


\bibitem{Stojanov}
P. Stojanov, Z. Li, M. Gong, Ruichu Cai, J. G. Carbonell, K. Zhang.
\newblock  Domain Adaptation with Invariant Representation Learning: What Transformations to Learn?
\newblock \emph{In NeurIPS}, 2021.

\bibitem{Zhang}
Y. Zhang, H. Tang, K. Jia, M. Tan.
\newblock  Domain-Symmetric Networks for Adversarial Domain Adaptation.
\newblock \emph{In CVPR}, 2019.

\bibitem{Liu}
J. Liu, Z. Hu, P. Cui, B. Li, Z. Shen.
\newblock  Heterogeneous Risk Minimization.
\newblock \emph{In ICML}, 2021.


\bibitem{Liu_2}
J. Liu, Z. Hu, P. Cui, B. Li, Z. Shen.
\newblock  Kernelized  Heterogeneous Risk Minimization.
\newblock \emph{In NeurIPS}, 2021.

\bibitem{Creager}
J. Creager, J. Jacobsen, R. Zemel.
\newblock  Environment Inference for Invariant Learning.
\newblock \emph{In ICML}, 2021.

\bibitem{Parascandolo}
G. Parascandolo, A. Neitz, A. Orvieto, L. Gresele, B. Scholkopf.
\newblock  Learning explanations that are hard to vary.
\newblock \emph{In ICLR}, 2021.

\bibitem{Diederik}
D. P. Kingma, J. L. Ba.
\newblock  Adam: A Method for  Stochastic Optimization.
\newblock \emph{In ICLR}, 2015.

\bibitem{Wah}
C. Wah, S. Branson, P. Welinder, P. Perona, and S. Belongie,
\newblock The Caltech-UCSD Birds-200-2011 dataset.
\newblock Technical report, California Institute of Technology, 2011.

\bibitem{Zhou}
B. Zhou, A. Lapedriza, A. Khosla, A. Oliva, and A. Torralba.
\newblock Places: A 10 million image database for scene recognition.
\newblock \emph{IEEE Transactions on Pattern Analysis and Machine Intelligence}, 40(6): 1452--1464, 2017

\bibitem{Rojas}
M. Rojas-Carulla, B. Scholkopf, R. Turner, and J. Peters.
\newblock Invariant models for causal transfer learning.
\newblock \emph{Journal of Machine Learning Research}, 2018


\bibitem{Andrychowicz1}
M. Andrychowicz1, M. Denil1, S. Colmenarejo1, M. W. Hoffman1, D. Pfau1, T. Schaul, B. Shillingford, N. de Freitas.
\newblock Learning to learn by gradient descent by gradient descent.
\newblock \emph{In NIPS}, 2016

\bibitem{Pham}
H. Pham, Z. Dai, Q. Xie, Q. V. Le.
\newblock Meta Pseudo Labels.
\newblock \emph{In CVPR}, 2021

\bibitem{Lee}
D. Lee.
\newblock Pseudo-Label : The Simple and Efficient Semi-Supervised Learning Method for Deep Neural Networks.
\newblock \emph{In ICML Workshop}, 2013 

\bibitem{Zheng}
Z, Zheng, L. Zheng, Y. Yang.
\newblock Unlabeled Samples Generated by GAN Improve the Person Reidentification Baseline in vitro,
 \newblock \emph{In ICCV}, 2017
 
 
\bibitem{Gu}
X. Gu, J. Sun, Z. Xu.
\newblock Spherical Space Domain Adaptation With Robust Pseudo-Label Loss.
 \newblock \emph{In CVPR}, 2020
 
\bibitem{Cour} 
 T. Cour, B. Sapp, and B. Taskar. \newblock Learning from partial labels.
\newblock \emph{Journal of Machine Learning Research}, 2011.


\bibitem{Xu} 
 N. Xu, J. Lv, X.Geng.
 \newblock Partial label learning via label enhancement.
 \newblock \emph{ In AAAI},
2019.

\bibitem{Yan} 
 Y. Yan, Y. Guo.
 \newblock Partial Label Learning with Batch Label Correction.
 \newblock \emph{ In AAAI},
2021.

\bibitem{Ishida}
T. Ishida, G. Niu, A. Menon, M. Sugiyama.
 \newblock Complementary-Label Learning for Arbitrary Losses and Models.
 \newblock \emph{ In ICML},
2019.


\bibitem{Feng}
L. Feng, T. Kaneko, B. Han, Ga. Niu, B. An, M. Sugiyama.
 \newblock Learning with Multiple Complementary Labels.
 \newblock \emph{ In ICML},
2020.



\bibitem{Katsura}
Y. Katsura, M. Uchida. 
 \newblock Bridging Ordinary-Label Learning and Complementary-Label Learning.
 \newblock \emph{ In ACML},
2020.


\bibitem{Mateo}
M. Rojas-Carulla, B. Scholkopf, R. Turner, and J. Peters.
\newblock  Invariant models for causal transfer learning.
\newblock \emph{JMLR}, 2018.


    \bibitem{Diederik}
D. P. Kingma, J. L. Ba.
\newblock  Adam: A Method for  Stochastic Optimization.
\newblock \emph{In ICLR}, 2015.

\end{thebibliography}
\end{document}